\documentclass[10pt,twocolumn,letterpaper]{article}

\usepackage{cvpr}
\usepackage{times}
\usepackage{epsfig}
\usepackage{graphicx}
\usepackage{amsmath}
\usepackage{amssymb}
\usepackage{amsthm}
\usepackage{nicefrac}       

\usepackage{url}
\usepackage{graphicx}
\usepackage{comment}
\usepackage{wrapfig}
\usepackage{caption}
\usepackage{subcaption}
\usepackage{bm}
\usepackage{appendix}

\newcommand\E{\mathbb{E}}

\newcommand\B{\mathcal{B}}
\newcommand\D{\mathcal{D}}
\newcommand\KL{\text{KL}}

\newtheorem{theorem}{Theorem}[section]

\newtheorem{lemma}[theorem]{Lemma}

\usepackage{multirow}
\usepackage{array}
\newcommand{\PreserveBackslash}[1]{\let\temp=\\#1\let\\=\temp}
\newcolumntype{C}[1]{>{\PreserveBackslash\centering}p{#1}}
\newcolumntype{R}[1]{>{\PreserveBackslash\raggedleft}p{#1}}
\newcolumntype{L}[1]{>{\PreserveBackslash\raggedright}p{#1}}

\usepackage[pagebackref=true,breaklinks=true,letterpaper=true,colorlinks,bookmarks=false]{hyperref}

\cvprfinalcopy 

\def\cvprPaperID{xxxx} 

\ifcvprfinal\fi
\begin{document}

\title{S3VAE: Self-Supervised Sequential VAE \\for Representation Disentanglement and  Data Generation}

\author{Yizhe Zhu$^{1,2}$, \quad Martin Renqiang Min$^{1}$, \quad Asim Kadav$^{1}$, \quad Hans Peter Graf$^{1}$\\
	  yizhe.zhu@rutgers.edu,  \quad \{renqiang, asim, hpg\}@nec-labs.com
	  \\
	$^{1}$NEC Labs America, $^{2}$Department of Computer Science, Rutgers University
}

\maketitle

\begin{abstract}
    We propose a sequential variational autoencoder to learn disentangled representations of sequential data (e.g., videos and audios) under self-supervision. Specifically, we exploit the benefits of some readily accessible supervisory signals from input data itself or some off-the-shelf functional models and accordingly design auxiliary tasks for our model to utilize these signals. With the supervision of the signals, our model can easily disentangle the representation of an input sequence into static factors and dynamic factors (i.e., time-invariant and time-varying parts). Comprehensive experiments across videos and audios verify the effectiveness of our model on representation disentanglement and generation of sequential data, and demonstrate that, our model with self-supervision performs comparable to, if not better than, the fully-supervised model with ground truth labels, and outperforms state-of-the-art unsupervised models by a large margin.
\end{abstract}

\section{Introduction}
Representation learning is one of the essential research problems in machine learning and computer vision~\cite{bengio2013representation}. Real-world sensory data such as videos, images, and audios are often in the form of high dimensions. Representation learning aims to map these data into a low-dimensional space to make it easier to extract semantically meaningful information for downstream tasks such as classification and detection. Recent years have witnessed rising interests in disentangled representation learning, which tries to separate the underlying factors of observed data variations such that each factor exclusively interprets one type of semantic attributes of sensory data. 
The representation of sequential data is expected to be disentangled into time-varying factors and time-invariant factors. For video data, the identity of a moving object in a video is regarded as time-invariant factors, and the motion in each frame is considered as time-varying ones ~\cite{li2018disentangled}. For speech data, the representations of
the timbre of speakers and the linguistic contents are expected to be disentangled~\cite{hsu2017unsupervised}. There are several benefits
of learning disentangled representations. First, the models
that produce disentangled representations are more explainable. Second, disentangled representations make it easier and more efficient to manipulate data generation, which has potential applications in entertainment industry, training data synthesis~\cite{zhu2018generative, zhu2019learning} and several downstream tasks~\cite{kay2017kinetics, Fan2019VideoSal, fan2018SOC, ren2015faster, zhu2019semantic}.

\begin{figure}
	\begin{center}
		\includegraphics[width=0.45\textwidth]{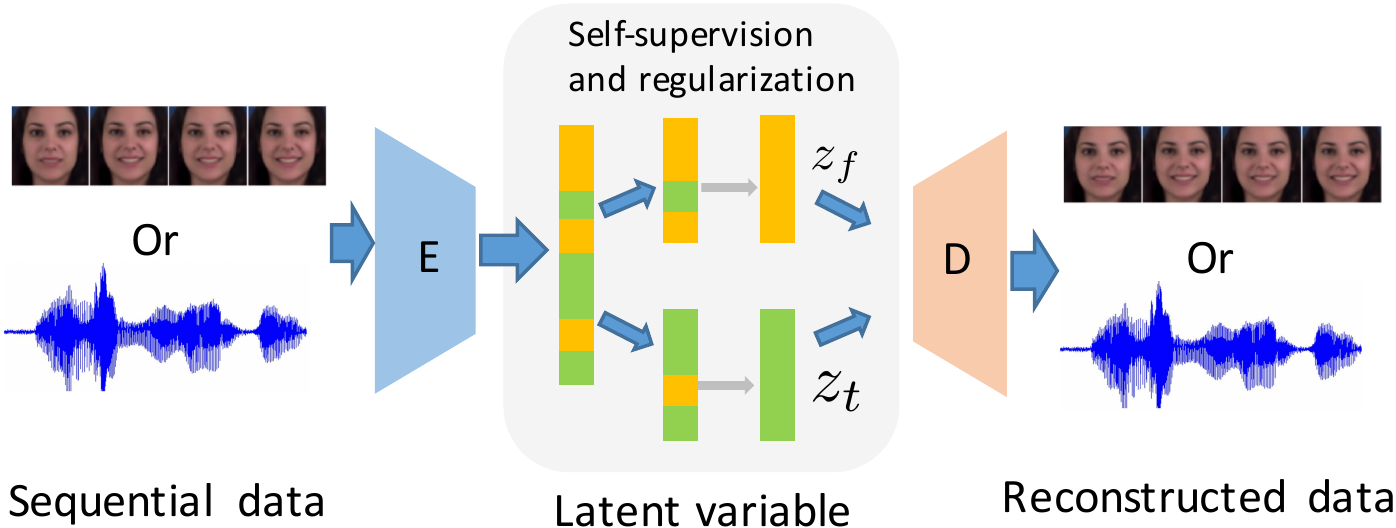}
	\end{center}
	\label{framework}
	\caption{Self-supervision and regularizations enforce the latent variable of our sequential VAE to be disentangled into a static representation $\bm{z}_f$ and a dynamic representation $\bm{z}_t$.   }
	\vspace{-1.5em}
\end{figure}

Despite the vast amount of works~\cite{Higgins2017betaVAELB, kim2018disentangling,chen2018isolating,esmaeili2018structured,esmaeili2018hierarchical,chen2016infogan,jeon2018ib} on disentangled representations of static data (mainly image data), fewer works~\cite{hsu2017unsupervised,li2018disentangled,he2018probabilistic,sun2018two} have explored representation disentanglement for sequential data generation. 
For unsupervised models,  FHVAE~\cite{hsu2017unsupervised} and DSVAE~\cite{li2018disentangled} elaborately designed model architectures and factorized latent variables into static and dynamic parts.
These models may well handle simple data forms such as synthetic animation data but fail when dealing with realistic ones as we will show later. 
Besides, as pointed out in \cite{locatello2018challenging}, unsupervised representation disentanglement is impossible without inductive biases. Without any supervision, the performance of disentanglement can hardly be guaranteed and greatly depends on the random seed and the dimensionality of latent vectors set in the models.
On the other hand, several works~\cite{he2018probabilistic,sun2018two} resort to utilizing label information or attribute annotation as strong supervision for disentanglement. For instance, 
VideoVAE~\cite{he2018probabilistic} leveraged  holistic attributes to constrain latent variables. Nevertheless, the costly annotation of data is essential for these models and prevents them from being deployed to most real-world applications, in which a tremendous amount of unlabeled data is available. 

To alleviate the drawbacks of both unsupervised and supervised models discussed above, this work tackles representation disentanglement for sequential data generation utilizing self-supervision. 
In self-supervised learning, various readily obtainable supervisory signals have been explored for representation learning of images and videos, employing auxiliary data such as the ambient sounds in videos~\cite{owens2016ambient,arandjelovic2017look}, the egomotion of cameras~\cite{agrawal2015learning,jayaraman2015learning}, the geometry cue in 3D movies~\cite{gan2018geometry}, and off-the-shelf functional models for visual tracking~\cite{wang2015unsupervised}, and  optical flow~\cite{pathak2017learning, Wang_2019_CVPR}. However, how self-supervised learning benefits representation disentanglement of sequential data has barely been explored.

In this paper, we propose a sequential variational autoencoder (VAE), a recurrent version of VAE, for sequence generation. In the latent space, the representation is disentangled into time-invariant and time-varying factors. We address the representation disentanglement by exploring intrinsic supervision signals, which can be readily obtained from both data itself and off-the-shelf methods, and accordingly design a series of auxiliary tasks.  
Specifically, on one hand, to exclude dynamic information from time-invariant variable, we exploit the temporal order of the sequential data and expect the time-invariant variable of the temporally shuffled data to be close to if not the same as that of the original data.
On the other hand, the time-varying variable is expected to contain dynamic information in different modalities. For video data, we allow it to predict the location of the largest motion in every frame, which can be readily inferred from optical flow. For audio data, the volume in each segment as an intrinsic label is leveraged as the supervisory signal. To further encourage the representation disentanglement, the mutual information between static and dynamic variables are minimized as an extra regularization. 

To the best of our knowledge, this paper is the first work to explicitly use auxiliary supervision to improve the representation disentanglement for sequential data. Extensive experiments on representation disentanglement and sequence data generation demonstrate that, with these multiple freely accessible supervisions, our model dramatically outperforms unsupervised learning-based methods and even performs better than fully-supervised learning-based methods in several cases. 

\section{Related Work}
\noindent\textbf{Disentangled Sequential Data Generation}
With the success of deep generative models, recent works~\cite{Higgins2017betaVAELB, kim2018disentangling,chen2018isolating, chen2016infogan,jeon2018ib} resort to variational autoencoders (VAEs)~\cite{kingma2013auto} and generative adversarial networks (GANs)~\cite{goodfellow2014generative} to learn a disentangled representation. Regularizations are accordingly designed.  $\beta$-VAE~\cite{Higgins2017betaVAELB} imposed a heavier penalty on the KL divergence term for a better disentanglement learning. Follow-up researches~\cite{kim2018disentangling, chen2018isolating} derived a Total Correlation (TC) from the KL term, and highlights this TC term as the key factor in  disentangled representation learning. In InfoGAN~\cite{chen2016infogan}, the disentanglement of a latent code $c$ is achieved by maximizing a mutual information lower-bound between $c$ and the generated sample $\tilde{x}$.

Several works involving disentangled representation have been proposed for video prediction.  Villegas \emph{et al.}~\cite{villegas2017decomposing} and Denton \emph{et al.}~\cite{denton2017unsupervised} designed two networks to encode pose and content separately at each timestep. 
Unlike video prediction, video generation from priors, which we perform in this work, is a much harder task since no frame is available for appearance and motion modeling in the generation phase. 

To handle video generation, VAEs are extended to a recurrent version~\cite{fabius2014variational, bayer2014learning, chung2015recurrent}.
However, these models do not explicitly consider static and dynamic representation disentanglement and fail to perform manipulable data generation. More recently, several works have proposed VAEs with factorized latent variables. 
FHVAE~\cite{hsu2017unsupervised} presented a factorized hierarchical graphical model that imposes sequence-dependent priors and sequence-independent priors to different sets
of latent variables in the context of speech data, but did not take advantage of the sequential prior.   Combining the merits of recurrent VAE and FHVAE,
DSVAE \cite{li2018disentangled} is capable of disentangling latent factors by factorizing them into time-invariant and time-dependent parts and applies an LSTM sequential prior to keep a better sequential consistency for sequence generation. Although with elaborately designed complex architectures, these models may only perform decently on representation disentanglement of simple data, the disentanglement performance degrades rapidly when the complexity of data increases.
In contrast, our work explores both model and regularization designs for representation disentanglement and sequential data generation. Our model fully factorizes the latent variables to time-invariant and time-varying parts, and both the posterior and the prior of the time-varying variable are modeled by LSTM for dynamic consistency. The auxiliary tasks with readily accessible supervisory signals are designed to regularize and encourage representation disentanglement.

\begin{figure*}
	\centering
	\includegraphics[width=0.78\textwidth]{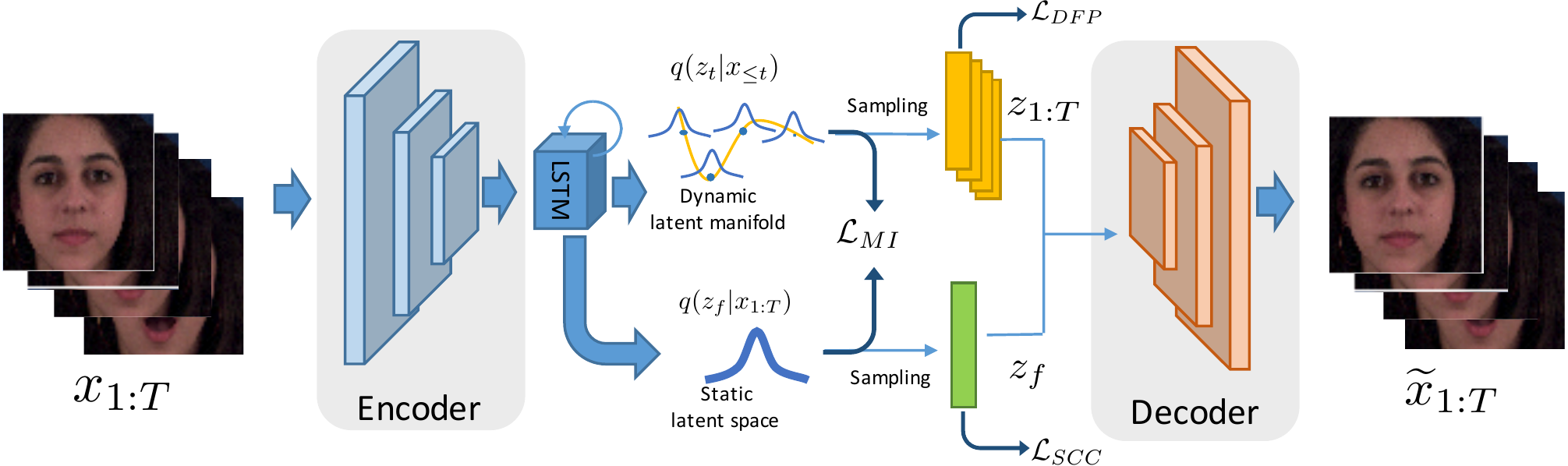}
	\vspace{-0.8em}
	\caption{The framework of our proposed model in the context of video data. Each frame of a video $\bm{x}_{1:T}$ is fed into an encoder to produce a sequence of visual features, which is then passed through an LSTM module to obtain the manifold posterior of a dynamic latent variable $\{q(\bm{z}_t|\bm{x}_{\leq t})\}_{t=1}^T$ and the posterior of a static latent variable $q(\bm{z}_f|\bm{x}_{1:T})$. The static and dynamic representations $\bm{z}_f$ and $\bm{z}_{1:T}$ are sampled from the corresponding posteriors and concatenated to be fed into a decoder to generate reconstructed sequence $\widetilde{\bm{x}}_{1:T}$. Three regularizers are imposed on dynamic and static latent variables to encourage the representation disentanglement. } 
	\vspace{-1.7em}
		\label{framework}
\end{figure*}

\noindent\textbf{Self-Supervised Learning}
The concept of self-supervised learning traces back to the autoencoder~\cite{hinton1994autoencoders}, which uses the input itself as supervision to learn the representation. Denoising autoencoder\cite{vincent2008extracting} makes the learned representations robust to noise and partial corruption of the input pattern by adding noise to the input. Recent years have witnessed the booming interest in self-supervised learning. The sources of supervisory signals can be roughly categorized into three classes. (a) \textit{\textbf{Intrinsic labels}}:  Doersch \emph{et al.}~\cite{doersch2015unsupervised} explored the use of spatial context in images, and Noroozi \emph{et al.}~\cite{noroozi2016unsupervised} trained a model to solve Jigsaw puzzles as a pretext task. Several works~\cite{zhang2016colorful, larsson2017colorization} showed that colorizing a gray-scale photograph can be utilized as a powerful pretext task for visual understanding. Temporal information of video is another  readily accessible supervisory signal. \cite{misra2016shuffle} trained a model to determine whether a sequence of frames from a video is in the correct temporal order and \cite{jing2019self} made the model learn to arrange the permuted 3D spatiotemporal crops.
(b) \textit{\textbf{Auxiliary data}}:  
Agrawal~\emph{et al.}~\cite{agrawal2015learning} and Jayaraman~\emph{et al.}~\cite{jayaraman2015learning} exploited the freely available knowledge of camera motion as a supervisory signal for feature learning. Ambient sounds in videos~\cite{owens2016ambient, arandjelovic2017look} are used as a supervisory signal for learning visual models.  The geometry cue in 3D movies~\cite{gan2018geometry} is utilized for visual representation learning. (c) \textit{\textbf{Off-the-shelf tools}}:  Wang \emph{et al.}~\cite{wang2015unsupervised} leveraged the visual consistency of objects from a visual tracker in the video clips. \cite{pathak2017learning} used segments obtained by  motion-based segmentation based on optical flow as pseudo ground truth for the single-frame object segmentation.
Instead of learning visual features as in aforementioned methods, this work aims to achieve static and dynamic representation disentanglement for sequential data such as video and speech. To this end, we leverage supervisory signals from intrinsic labels to regularize the static representation and off-the-shelf tools to regularize the dynamic representation.

\section{Sequential VAE Model}

We start by introducing some notations and the problem definition. 
$\mathcal{D} = \{\bm{X}^i\}^N$ is given as a dataset that consists of $M$ i.i.d. sequences, where $\bm{X} \equiv\bm{x}_{1:T} = (x_1, x_2, ... x_T)$ denotes a sequence of $T$ observed variables, such as a video of $T$ frames or an audio of $T$ segments.   
We propose a sequential variational autoencoder model, where the sequence is assumed to be generated from latent variable $\bm{z}$ and $\bm{z}$ is factorized into two disentangled variables:  the time-invariant (or static) variable $\bm{z}_f$ and the time-varying (or dynamic) variables $\bm{z}_{1:T}$. 

\noindent\textbf{{Priors}} \quad
The prior of $\bm{z}_f$ is defined as a standard Gaussian distribution: $\bm{z}_f \sim \mathcal{N}(0, 1)$.  The time-varying latent variables $\bm{z}_{1:T}$ follow a sequential prior
$\bm{z}_t \hspace{0.3em}|\hspace{0.3em} \bm{z}_{< t} \sim \mathcal{N}( \bm{\mu}_t,  \text{diag}(\bm{\sigma}^2_t))$, 
where $[\bm{\mu}_t, \bm{\sigma}_t] = \phi^{prior}_{R}(\bm{z}_{<t})$, $\bm{\mu}_t, \bm{\sigma}_t$ are the parameters of the prior distribution conditioned on all previous time-varying latent variables $\bm{z}_{<t}$.  The model $\phi^{prior}_{R}$ can be parameterized as a recurrent network, such as LSTM~\cite{hochreiter1997long} or GRU~\cite{cho2014learning}, where the hidden state is updated temporally. The prior of $\bm{z}$ can be factorized as:
\vspace{-0.5em}
\begin{equation}
p(\bm{z}) = p(\bm{z}_f) p(\bm{z}_{1:T})= p(\bm{z}_f)\prod_{t=1}^{T}p(\bm{z}_t|\bm{z}_{<t}).
\end{equation}
\noindent\textbf{Generation} 
The generating distribution of time step $t$ is conditioned on $\bm{z}_f$ and $\bm{z}_t$:
$\bm{x}_t\hspace{0.3em}|\hspace{0.3em} \bm{z}_f, \bm{z}_t \sim \mathcal{N}( \bm{\mu}_{x,t},  \text{diag}(\bm{\sigma}^2_{x,t}))$,
where $[\bm{\mu}_{x,t}, \bm{\sigma}_{x,t}] = \phi^{Decoder}(\bm{z}_{f}, \bm{z}_{t})$ and the decoder $\phi^{Decoder}$ can be a highly flexible function such as a deconvolutional neural network~\cite{noh2015learning}. 

The complete generative model can be formalized by the factorization:
\vspace{-0.4em}
\begin{equation}
p(\bm{x}_{1:T}, \bm{z}_{1:T}, \bm{z}_f ) =   p(\bm{z}_f) \prod^T_{t=1} p(\bm{x}_t| \bm{z}_f, \bm{z}_t) p(\bm{z}_t|\bm{z}_{<t})
\end{equation}
\noindent\textbf{Inference} \quad
Our sequential VAE uses variational inference to approximate posterior distributions:
\begin{equation}
\bm{z}_f \sim \mathcal{N}( \bm{\mu}_f,  \text{diag}(\sigma^2_f)), \quad \bm{z}_t \sim \mathcal{N}( \bm{\mu}_t,  \text{diag}(\sigma^2_t)), 
\end{equation}
where $[\bm{\mu}_f, \bm{\sigma}_f] = \psi^{Encoder}_{f}(\bm{x}_{1:T})$ and  $[\bm{\mu}_t, \bm{\sigma}_t] = \psi^{Encoder}_{R}(\bm{x}_{\leq t})$. The static variable $\psi^{Encoder}_{f}$ is conditioned on the whole sequence while the dynamic variable is inferred by a recurrent encoder $\psi^{Encoder}_{R}$  and only conditioned on the previous frames. Our inference model is factorized as:
\begin{equation}
q(\bm{z}_{1:T}, \bm{z}_f|\bm{x}_{1:T}) = 
q(\bm{z}_f|\bm{x}_{1:T})\prod^T_{t=1} q(\bm{z}_t|\bm{x}_{\leq t})
\end{equation}

\noindent\textbf{Learning} \quad
The objective function of sequential VAE is a timestep-wise negative variational lower bound:
\begingroup\makeatletter\def\f@size{9}\check@mathfonts
\def\maketag@@@#1{\hbox{\m@th\large\normalfont#1}}
\begin{equation}
\begin{split}
\mathcal{L}_{VAE} &= \E_{q(\bm{z}_{1:T}, \bm{z}_f|\bm{x}_{1:T})}[-\sum_{t=1}^{T}\log p(\bm{x}_t|\bm{z}_{f},\bm{z}_{t})] +\\ &\KL(q(\bm{z}_f|\bm{x}_{1:T})||p(\bm{z}_f)) 
+ \sum_{t=1}^{T}\KL(q(\bm{z}_t|\bm{x}_{\leq t})||p(\bm{z}_t|\bm{z}_{< t}))
\end{split}
\end{equation}\endgroup

The schematic representation of our model is shown in Figure~\ref{framework}. Note that DSVAE also proposes a sequential VAE with disentangled representation, but it either independently infers $\bm{z}_t$ only based on the frame of each time-step without considering the continuity of dynamic variables and thus may generate inconsistent motion, or assumes the variational posterior of $\bm{z}_{t}$ depends on $\bm{z}_f$, implying that the variables are still implicitly entangled. In contrast, we model both the prior and the posterior of $\bm{z}_t$ by recurrent models independently, resulting in consistent dynamic information in synthetic sequences, and ensure full disentanglement of $\bm{z}_f$ and $\bm{z}_t$ by posterior factorization.

\section{Self-Supervision and Regularization}
Without any supervision, there is no guarantee that the time-invariant representation $\bm{z}_f$ and the time-varying representation $\bm{z}_t$ are disentangled.  In this section, we introduce a series of auxiliary tasks on the different types of representations as the regularization of our sequential VAE to achieve the disentanglement, where readily accessible supervisory signals are leveraged.

\subsection{Static Consistency Constraint}
To encourage the time-invariant representation $\bm{z}_f$ to exclude any dynamic information, we expect that $\bm{z}_f$ changes little when varying dynamic information dramatically.  To this end, we shuffle the temporal order of frames to form a shuffled sequence.  Ideally,  the static factors of the original sequence and shuffled sequence should be very close, if not equal, to one another. However, directly minimizing the distance of these two static factors will lead to very trivial solutions, e.g., the static factors of all sequences converge to the same value and do not contain any meaningful information. Thus, we randomly sample another sequence as the negative sample of the static factor. With a triple of static factors, we introduce a triplet loss as follows: 
\begin{equation}
\mathcal{L}_{SCC} = \max \left( D(\bm{z}_f, \bm{z}_f^{pos})- D(\bm{z}_f, \bm{z}_f^{neg})+ \bm{m}, 0 \right),
\end{equation}
where $\bm{z}_f$, $\bm{z}_f^{pos}$ and $\bm{z}_f^{neg}$ are the static factors of the anchor sequence, the shuffled sequence as the positive data, and another randomly sampled video as the negative data, $D(\cdot,\cdot)$ denotes the Euclidean distance and $\bm{m}$ is the margin, set to 1. This regularization makes $\bm{z}_f$ preserve meaningful static information to a certain degree while excluding dynamic information. 

\begin{figure}
	\centering 
	\includegraphics[width=0.8\columnwidth]{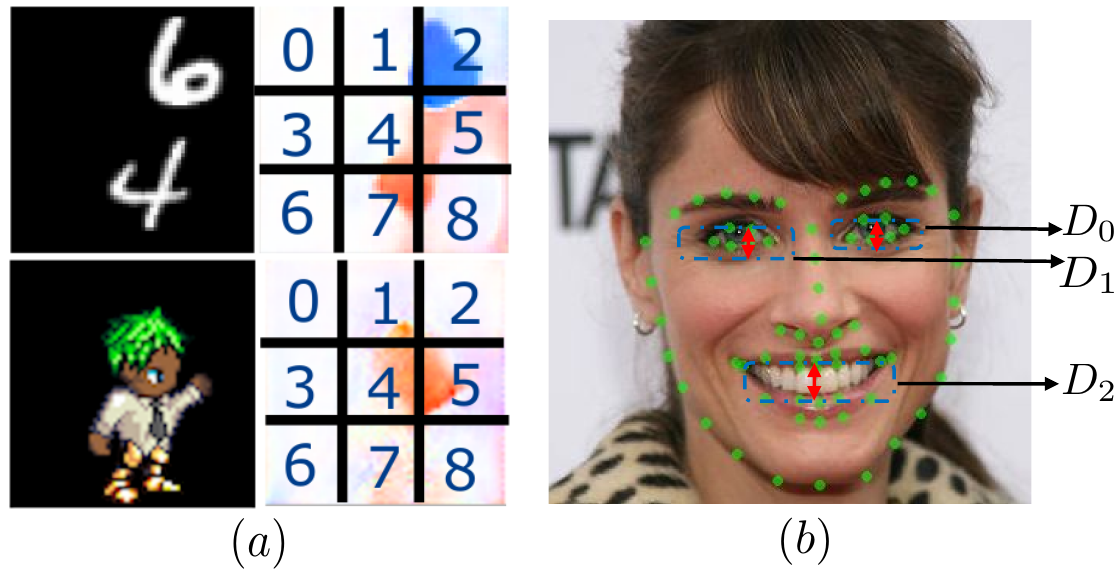}
	\vspace{-0.5em}
	\caption{The pseudo label generation for video datasets. \textbf{(a)} The left image is the input frame and the right image is the corresponding optical flow map that is split by a grid. \textbf{(b)} Three distances are used as the dynamic signals for the face dataset.}
	\label{video_of}
	\vspace{-1em}
\end{figure}

\subsection{Dynamic Factor Prediction}
To encourage the dynamic representation $\bm{z}_t$ to carry adequate and correct time-dependent information of each time-step, we exploit dynamic information-related signals from  off-the-shelf tools for different types of sequential data and accordingly design the auxiliary tasks as the regularization  imposed on $\bm{z}_t$.  We have the loss $\mathcal{L}_{DFP} = \mathcal{L}(\phi_a(\bm{z}_t), y)$, where $\mathcal{L}(\cdot, \cdot)$ can be either cross-entropy loss or mean squared error loss according to the designed auxiliary task, $\phi_a(\cdot)$ is a network for dynamic factor prediction and $y$ contains supervisory signals. 

\noindent\textbf{Video Data} \quad
The dynamic representation $\bm{z}_t$ can be learned by forcing it to predict the dynamic factors of videos. Motivated by this, we expect the location of the largest motion regions can be  accurately predicted based on $\bm{z}_t$. To this end,  the optical flow maps of video are first obtained by commonly used functional model FlowNet2~\cite{ilg2017flownet} and then split into patches by grid, as shown in Figure~\ref{video_of}.a. We compute the average of motion magnitudes for every patch and use the indices of patches with the top-k largest values as the pseudo label for prediction. For this task, $\phi_a(\cdot)$ is implemented with two fully-connected layers and a softmax layer.

Apart from the optical flow, some freely obtainable data-specific signals can be exploited. For a human face dataset, the landmark of each frame can be readily detected and considered as a supervision for dynamic factors. We obtain the landmark from an off-the-shelf landmark detector~\cite{dong2018style}.
To keep our model concise and efficient, we only leverage the distance between the upper and lower eyelids as well as the distance between the upper and lower lips in each frame as the dynamic signal, as shown in Figure~\ref{video_of}.b. Here, $\phi_a(\cdot)$ consists of two fully-connected layers to regress the three distances. We observe that our model can easily capture dynamic motions under this simple supervision. 

\noindent\textbf{Audio Data} \quad
For the audio dataset, we consider the volume as  time-dependency factor and accordingly design an auxiliary task, where $\bm{z}_t$ is forced to predict if the speech is silent or not in each segment. The pseudo label is readily obtained by setting a magnitude threshold on the volume of each speech segment. $\phi_a(\cdot)$ consists of two fully-connected layers and performs a binary classification. 

\subsection{Mutual Information Regularization}
Forcing the time-varying variable $\bm{z}_t$ to predict dynamic factors can guarantee that $\bm{z}_t$  contains adequate dynamic information, but this fails to guarantee that $\bm{z}_t$  excludes the static information.  Therefore, we introduce the mutual information between static and dynamic variables as a regulator $\mathcal{L}_{MI}$.
The mutual information is a measure of the mutual dependence between two variables. By minimizing $\mathcal{L}_{MI}$, we encourage the information in these two variables are mutually exclusive. The mutual information is formally defined as the KL divergence of the joint distribution to the product of marginal distribution of each variable. We have
\begingroup\makeatletter\def\f@size{9}\check@mathfonts
\def\maketag@@@#1{\hbox{\m@th\large\normalfont#1}}
\begin{equation}
\begin{split}
\mathcal{L}_{{MI}}(\bm{z}_f, \bm{z}_{1:T}) 
&= \sum_{t=1}^T \text{KL}(q(\bm{z}_f, \bm{z}_t)|| q(\bm{z}_f)q(\bm{z}_t))\\
&= \sum_{t=1}^T [ H(\bm{z}_f) + H(\bm{z}_t) - H(\bm{z}_f, \bm{z}_t)], 
\end{split}
\end{equation}
\endgroup
where $H(\cdot) = -\E_{q(\bm{z})}[\log(q(\cdot))] = - \E_{q(\bm{z}_f, \bm{z}_t)}[\log(q(\cdot))]$. The expectation can be estimated by the mini-batch weighted sampling estimator~\cite{chen2018isolating}, 
\begingroup\makeatletter\def\f@size{9}\check@mathfonts
\def\maketag@@@#1{\hbox{\m@th\large\normalfont#1}}
\begin{equation}
\E_{q(\bm{z})}[\log q(\bm{z}_n)] \approx \frac{1}{M}\sum_{i=1}^M \left[ \log \sum_{j=1}^{M} q(\bm{z}_n(\bm{x}_i)|\bm{x}_j) - \log(NM) \right], 
\end{equation}\endgroup
for $\bm{z}_n = \bm{z}_f, \bm{z}_t$ or $(\bm{z}_f, \bm{z}_t)$, where $N$ and $M$ are the data size and the minibatch size, respectively. 

\subsection{Objective Function}
Overall, our objective can be considered as the sequential VAE loss with a series of self-supervision and regularization: 
\begin{equation}
\mathcal{L} = \mathcal{L}_{VAE} + \lambda_1\mathcal{L}_{SCC} + \lambda_2\mathcal{L}_{DFP} + \lambda_3\mathcal{L}_{MI},
\end{equation}
where $\lambda_1$, $\lambda_2$ and $\lambda_3$ are balancing factors.

\section{Experiments}
To comprehensively validate the effectiveness of our proposed model with self-supervision, we conduct experiments on three video datasets and one audio dataset. With these four datasets, we cover different modalities from video to audio. In video domain, a large range of motions are covered from large character motions (e.g., walking, stretching) to subtle facial expressions (e.g, smiling, disgust).  

\subsection{Experiments on Video Data}
We present an in-depth evaluation on two problems, tested on three different datasets and employing a large variety of metrics. 
  
\subsubsection{Datasets}
\textbf{Stochastic Moving MNIST} is introduced by~\cite{denton2018stochastic} and consists of sequences of $15$ frames of size $64 \times 64$, where two digits from  MNIST dataset move in random directions. We randomly generate 6000 sequences, 5000 of which are used for training and the rest are for testing.

\noindent\textbf{Sprite}~\cite{li2018disentangled} contains sequences of animated cartoon characters with 9 action categories: walking, casting spells and slashing with three viewing angles.  The appearance of characters are fully controlled by four attributes, i.e., the color of skin, tops, pants, and hair. Each of the attributes categories contains 6 possible variants, therefore it results in  totally $6^4 = 1296$ unique characters, 1000 of which are used for training and the rest for testing. Each sequence contains 8 frames of size $64 \times 64$.

\noindent\textbf{MUG Facial Expression}~\cite{5617662} consists of $3528$ videos with $52$ actors performing $6$ different facial expressions: anger, fear, disgust, happiness, sadness, and surprise. Each video composes of 50 to 160 frames. 
As suggested in MoCoGAN~\cite{tulyakov2018mocogan}, we crop the face regions, resize video to $64 \times 64$, and randomly sample a clip of $15$ frames in each video.  The $75\%$ of dataset is used for training and the rest for testing.    
\begin{figure}
	\centering 
	\includegraphics[width=1\columnwidth]{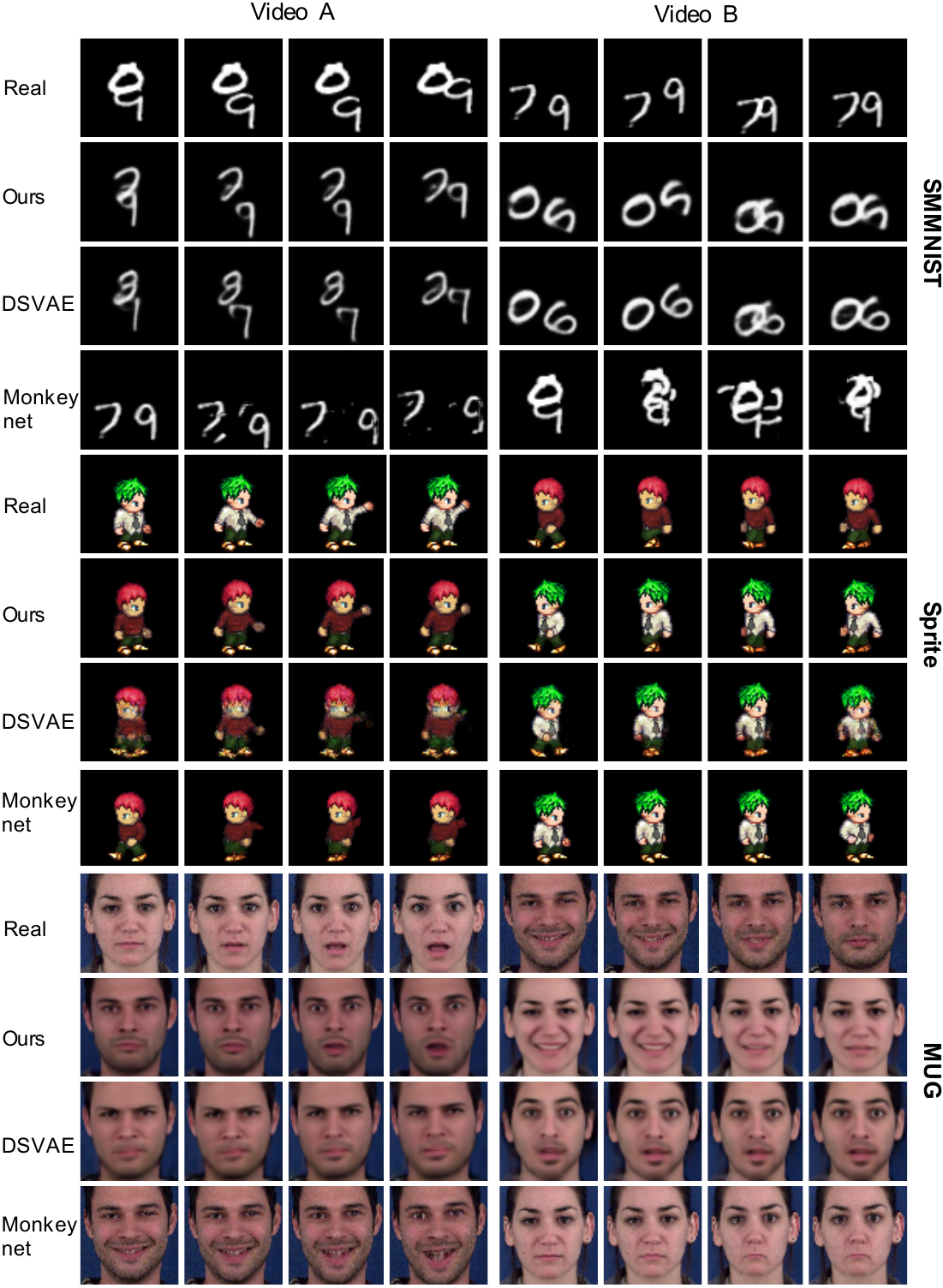}
	
	\caption{Representation swapping on SMMNIST, Sprite and MUG datasets. In each panel, we show two real videos as well as the generated videos by swapping $\bm{z}_f$ and  $\bm{z}_{1:T}$, from our model and two competing models: DSVAE and MonkeyNet. Each column is supposed to have the same motion.}
	\vspace{-1.5em}
	\label{swap}
\end{figure}

\subsubsection{Representation Swapping}
We first perform the representation swapping and compare our method with DSVAE, a disentangled VAE model, as well as MonkeyNet~\cite{Siarohin_2019_CVPR}, a state-of-the-art deformable video generation model.   Suppose two real videos are given for motion information and appearance information, denoted as $\bm{V}_m$ and $\bm{V}_a$. Our method and DSVAE perform video generation based on the $\bm{z}_f$ from $\bm{V}_a$ and $\bm{z}_{1:T}$ from $\bm{V}_m$ . For MonkeyNet, the videos are generated by deforming the first frame of $\bm{V}_a$ based on motion in $\bm{V}_m$.  The synthetic videos are expected to preserve the appearance in $\bm{V}_a$ and the motion in $\bm{V}_m$. The qualitative comparisons on three datasets are shown in Figure~\ref{swap}.

For SMMNIST, the generated videos of our model can preserve the identity of digits while consistently mimic the motion of the provided video. However, DSVAE can hardly preserve the identity. For instance, it mistakenly changes the digit ``9" to ``6". We observe that MonkeyNet can hardly handle the case with multiple objects like SMMNIST, because the case does not meet its implicit assumption of only one object moving in the video.

For Sprite, DSVAE generates blurry videos when the characters in $\bm{V}_a$ and $\bm{V}_m$ have opposite directions, indicating it fails to encode the direction information in the dynamic variable.  Conversely, our model can generate videos with the appearance of the character in $\bm{V}_a$ and  the same action and direction of the character in $\bm{V}_m$, due to the guidance from optical flow. 
 The characters in the generated videos of Monkeynet fail to follow the pose and action in $\bm{V}_m$, and many artifacts appear. E.g., an arm-like blob appears in the back of the character in the left panel. 

For MUG, the generated video of DSVAE can hardly preserve both the appearance in $\bm{V}_a$ and the facial expression in $\bm{V}_m$. For example, the person in the right has a mixed appearance characteristic, indicating $\bm{z}_f$ and $\bm{z}_t$ are entangled. Due to the deformation scheme of generation, MonkeyNet fails to handle the case where the faces in two videos are not well aligned. For instance, forcing the man with a smile to be fear results in unnatural expression. On the contrary, our model disentangles $\bm{z}_f$ and $\bm{z}_t$, supported by the realistic expressions on different faces in generated videos.

\begin{figure}
	\centering 
	\includegraphics[width=0.48\textwidth]{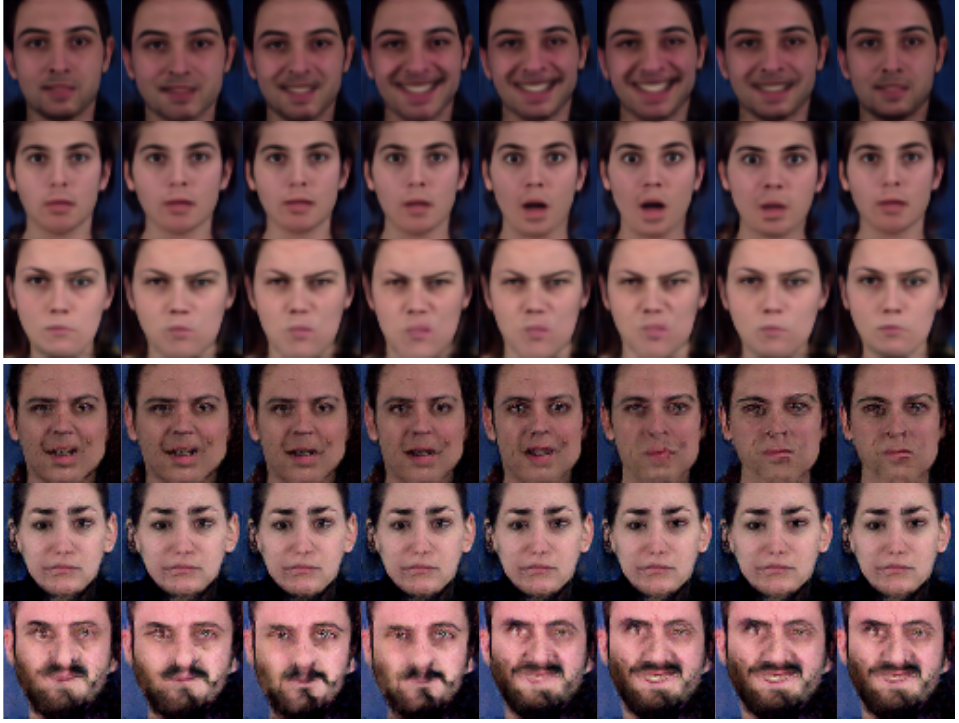}
	\vspace{-1.5em}
	\caption{Unconditional video generation on MUG. The upper and lower panels show the qualitative results of our model and MoCoGAN, respectively. }
	\vspace{-0.5em}
	\label{unconditional}
\end{figure}

\begin{figure}
	\centering 
	\includegraphics[width=0.5\textwidth]{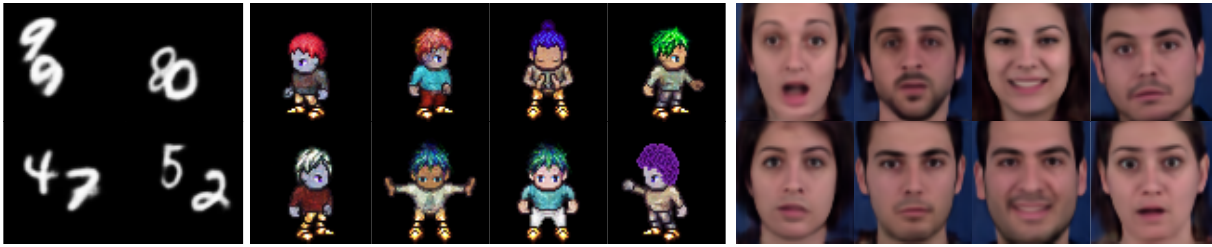}
	\caption{Randomly sampled frames for each dataset.}
	\label{sample}
	\vspace{-1em}
\end{figure}

 \begin{table*}[t]
 	\caption{Quantitatively performance comparison on SMMNIST, Sprite and MUG datasets. High values are expected for $Acc$, $H(y)$ and $IS$, while for $H(y|x)$, the lower values are better.	The results of our model with supervision of ground truth labels \textit{baseline-sv*} are shown as a reference.}
 	\vspace{-1em}
 	\label{disagree}
 	\centering
 	\scalebox{0.86}{
 		\begin{tabular}{|c|cccc|cccc|cccc|}\hline 
 			\multirow{2}{*}{Methods}
 			&\multicolumn{4}{c|}{SMMNIST} &\multicolumn{4}{c|}{Sprite} &\multicolumn{4}{c|}{MUG}\\
 			&$Acc$   &$IS$  & $H(y|x)$　&$H(y)$  &$Acc$   &$IS$  & $H(y|x)$　&$H(y)$ &$Acc$ &$IS$  & $H(y|x)$　&$H(y)$ \\\hline

 			MoCoGAN  &74.55\% &4.078&0.194 &0.191 &92.89\%  &8.461  &0.090 &2.192  &63.12\% &4.332 &0.183  &1.721 \\
 			DSVAE  &88.19\% &6.210&0.185&2.011 &90.73\%  &8.384 & 0.072 &2.192
 			&54.29\%  &3.608 &0.374 &1.657 \\\hline 
 			
		 	\textit{baseline}  &90.12\%&6.543&0.167 &2.052  &91.42\%  &8.312 &0.071 &2.190 &53.83\%  &3.736 &0.347 &1.717\\
 		
 			\textit{full model} &\textbf{95.09\%}&\textbf{7.072}&\textbf{0.150}&\textbf{2.106}
 			&\textbf{99.49\%}  &\textbf{8.637} &\textbf{0.041}  &\textbf{2.197} 
 			&\textbf{70.51\%}  &\textbf{5.136} &\textbf{0.135} & \textbf{1.760}

 			\\\hline\hline 
 			\textit{baseline-sv*} &92.18\%&6.845& 0.156&2.057  &98.91\%& 8.741&0.028&  2.196&72.32\%& 5.006 &0.129 &1.740\\\hline 
 		\end{tabular}
 	}
 \end{table*}
 
 \begin{figure*}
 	\begin{minipage}[c]{0.73\textwidth}
 		\includegraphics[width=\textwidth]{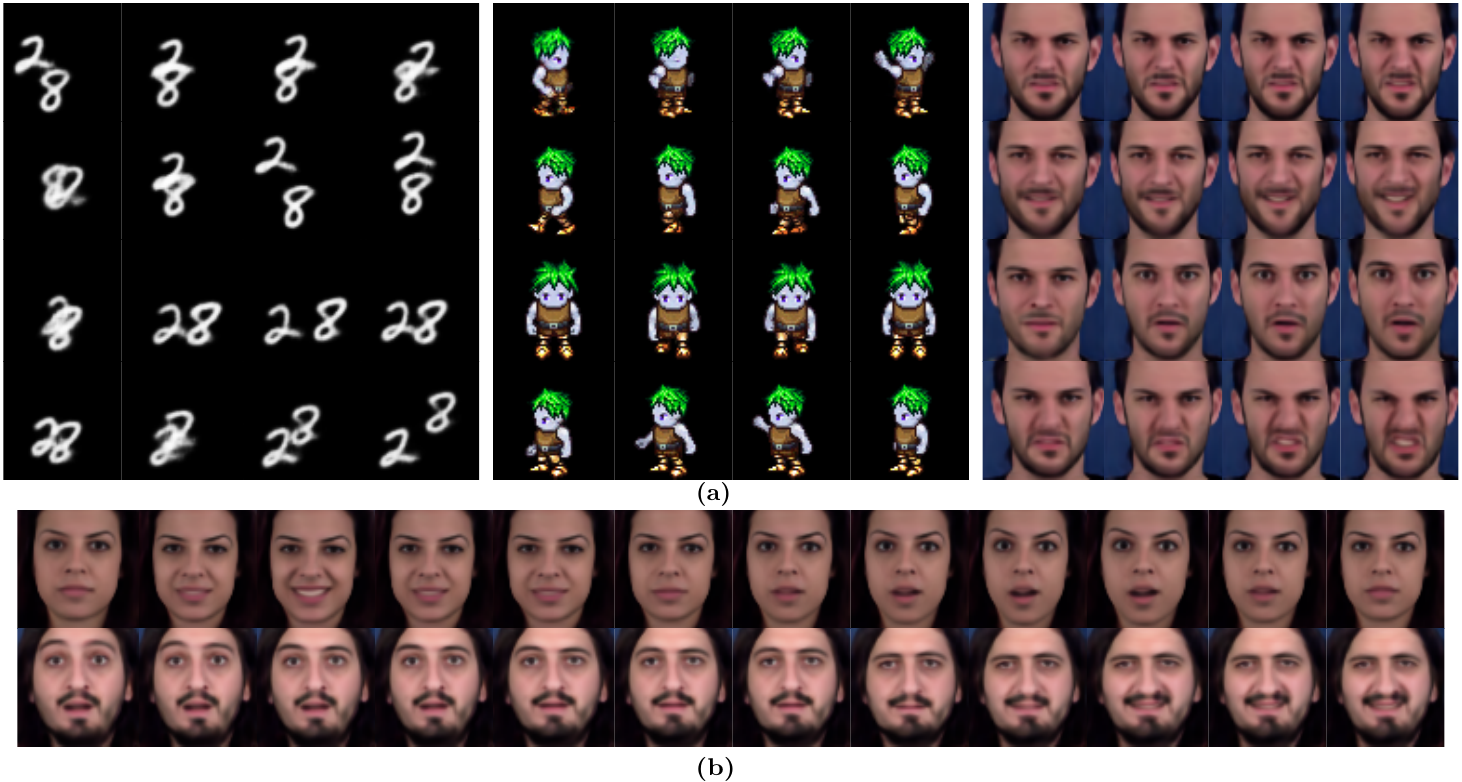}
 	\end{minipage}\hfill
 	\begin{minipage}[c]{0.25\textwidth}
 		\caption{
 			Controlled video generation. \textbf{(a)} Video generation controlled by fixing the static variable and randomly sampling dynamic variables from the prior $p(\bm{z}_{1:T})$. All sequences share a same identity but with different motions for each dataset. \textbf{(b)} Video generation with changed facial expressions. Expression is changed from smile to surprise and from surprise to disgust in two sequences, respectively. We control it by transferring the dynamic variables. 
 		} \label{control}
 	\end{minipage}
 \vspace{-1.5em}
 \end{figure*}

\subsubsection{Video Generation}
\noindent\textbf{Quantitative Results}
We compute the quantitative metrics of our model with and without self-supervsion and regularization, denoted as \textit{full model} and \textit{baseline}, as well as two competing methods: DSVAE and MoCoGAN. All these methods are comparable as no ground truth labels are used to benefit representation disentanglement. Besides, the results of our baseline with full supervision from human-annotation \textit{baseline-sv} are also provided as a reference. 

To demonstrate the ability of a model on the representation disentanglement, we use the classification accuracy $Acc$~\cite{li2018disentangled}, which measures the ability of  a model to preserve a specific attributes when generating a video given the corresponding representation or label.
To measure how diverse and realistic videos a model can generate, three metrics are used: $IS$~\cite{salimans2016improved}, Intra-Entropy $H(y|x)$~\cite{he2018probabilistic} and Inter-Entropy $H(y)$~\cite{he2018probabilistic}.
All metrics utilize a pretrained classifier based on the real videos and ground truth attributes.  See Appendix for the detailed definitions of the metrics. The results are shown in Table~\ref{disagree}.
 
For representation disentanglement, we consider generating videos with a given $\bm{z}_f$ inferred from a real video and randomly sampled $\bm{z}_{1:T}$ from the prior $p(\bm{z}_{1:T})$ for SMMNIST. We then check if the synthetic video contains the same digits as the real video by the pretrained classifier. For MUG, we evaluate the ability of a model to preserve the facial expression by fixing $\bm{z}_{1:T}$ and randomly sampled $\bm{z}_f$ from the prior $p(\bm{z}_f)$. For Sprite, since the ground truth of both actions and appearance attributes are available, we evaluate the ability of preserving both static and dynamic representations, and report the average scores.
It's evident that our \textit{full model} consistently outperforms all competing methods. For SMNIST, we observe that MoCoGAN have poor ability to correctly generate the digits with given labels  while our \textit{full model} can generate correctly digits, reflected by the high $Acc$. Note that our \textit{full model}  achieves $99.49\%$ $Acc$ on Sprite, indicating the $z_f$ and $z_t$ are greatly disentangled. 
Besides, \textit{full model} significantly boosts the performance of \textit{baseline}, especially in MUG where more realistic data is contained, the performance gets giant boost from $53.83\%$ to $70.51\%$, which illustrates the crucial role of our self-supervision and regularization. For video generation, \textit{full model} consistently shows the superior performances on $IS$, $H(y|x)$, $H(y)$.  Especially in MUG, \textit{full model} outperforms the runner-up MoCoGAN by $18.6\%$ on $IS$, demonstrating that high quality of videos generated by our model. Note that our \textit{baseline} is also compared favorably to DSVAE, illustrating the superiority of the designed sequential VAE model. 

It is worth noting that our model with self-supervision \textit{full model} outperforms \textit{baseline-sv} in SMMNIST on the representation disentanglement.　The possible reason is with the ground truth labels, \textit{baseline-sv} only encourages $\bm{z}_f$ to contain identity information but fails to exclude the information in $\bm{z}_t$, resulting in confusion in digit recognition when using various dynamic variables. For MUG, \textit{baseline-sv} performs better on preserving motion. We conjecture it's because that the strong supervision of expression labels forces $\bm{z}_t$ to encode the dynamic information, and $\bm{z}_f$ does not favor encoding temporal dynamic information and thus varying $\bm{z}_f$ does not affect the motion much.

\noindent\textbf{Qualitative results} We first demonstrate the ability of our model to manipulate video generation in Figure~\ref{control}. By fixing $\bm{z}_f$ and sampling $\bm{z}_{t}$, our model can generate videos with the same object that performs various motions as shown in Figure~\ref{control}.a.　 Even in one video, the facial expression can be transferred by controlling $\bm{z}_{t}$, as shown in Figure~\ref{control}.b. 

We also evaluate the appearance diversity of generate objects from our model. Figure~\ref{sample} shows the frames our model generates with sampled $\bm{z}_f$. The objects with realistic and diverse appearances validate our model's outstanding capability of high-quality video generation.

Besides, we compare our model with MoCoGAN on unconditional video generation on MUG dataset, as shown in Figure~\ref{unconditional}. The videos are generated with sampled $\bm{z}_f$ and $\bm{z}_t$. MoCoGAN generates videos with many artifacts, such as unrealistic eyes and inconsistent mouth in the third video. Conversely, our model generates more realistic human faces with consistent high-coherence expressions.

\subsubsection{Ablation Studies}
In this section, we present an ablation study to empirically measure the impact of each regularization of our model on its performance. The variant without a certain regularizaiton is denoted as $No$ $\mathcal{L}_{X}$.  
In Table~\ref{tb:ablation}, we report the quantitative evaluation. We note that $No$ $ \mathcal{L}_{SSC}$ performs worse than the full model. This illustrates the significance of the static consistency constraint to make $\bm{z}_f$ to be disentangled from $\bm{z}_t$ .
$No$ $\mathcal{L}_{DFP}$ degrades the performance on $Acc$ considerably, indicating that $\mathcal{L}_{DFP}$ as regularization is crucial to preserve the action information in the dynamic vector.  Besides, after removing the mutual information regularization, $No$ $\mathcal{L}_{MI}$ again shows an inferior performance to the full model. A possible explanation is that our $\mathcal{L}_{MI}$ encourages that $\bm{z}_t$ excludes the appearance information; thus the appearance information is only from $\bm{z}_f$. 
The qualitative results shown in Figure~\ref{ablation} confirms this analysis. We generate a video with $z_f$ of the woman and $z_t$ of the man in the first row by different variants of our model.  Without $\mathcal{L}_{MI}$, some characteristics of the man are still preserved, such as the beard, confirming that the appearance information partially remains in $z_t$.  In the results of $No$ $\mathcal{L}_{SSC}$, the beard is more evident, indicating that the static and dynamic variable are still entangled.   On the other hand, without  $\mathcal{L}_{DFP}$, the woman in the generated video cannot mimic the action of the man well, which indicates the dynamic variable does not encode the motion information properly. Finally, the person in the generated video of $baseline$ neither preserves the appearance of the woman nor follows the expression of the man. It illustrates the representation disentanglement without any supervision remains a hard task.

\begin{table}[t]
	\caption{Ablation study of disentanglement on MUG.}
	\label{tb:ablation}
	\centering 
	\vspace{-1em}
	\scalebox{0.9}{
		\begin{tabular}{|l|cccc|}\hline 
			Methods &$Acc$   &$IS$  & $H(y|v)$ &$H(y)$ 
			\\ \hline
			$No$ $\mathcal{L}_{SCC}$   &61.45\% &4.850 &0.201 &1.734\\
			$No$ $\mathcal{L}_{DFP}$   &58.32\%  &4.423 &0.284 &1.721\\
			$No$ $\mathcal{L}_{MI}$  &66.07\% & 4.874 &0.175  &1.749\\
			\textit{Full model}      &\textbf{70.51\%}  &\textbf{5.136} &\textbf{0.135} & \textbf{1.760}  \\\hline 
		\end{tabular}
	}
		\vspace{-1.5em}
\end{table}
\begin{figure}
	\centering 
	\includegraphics[width=0.38\textwidth]{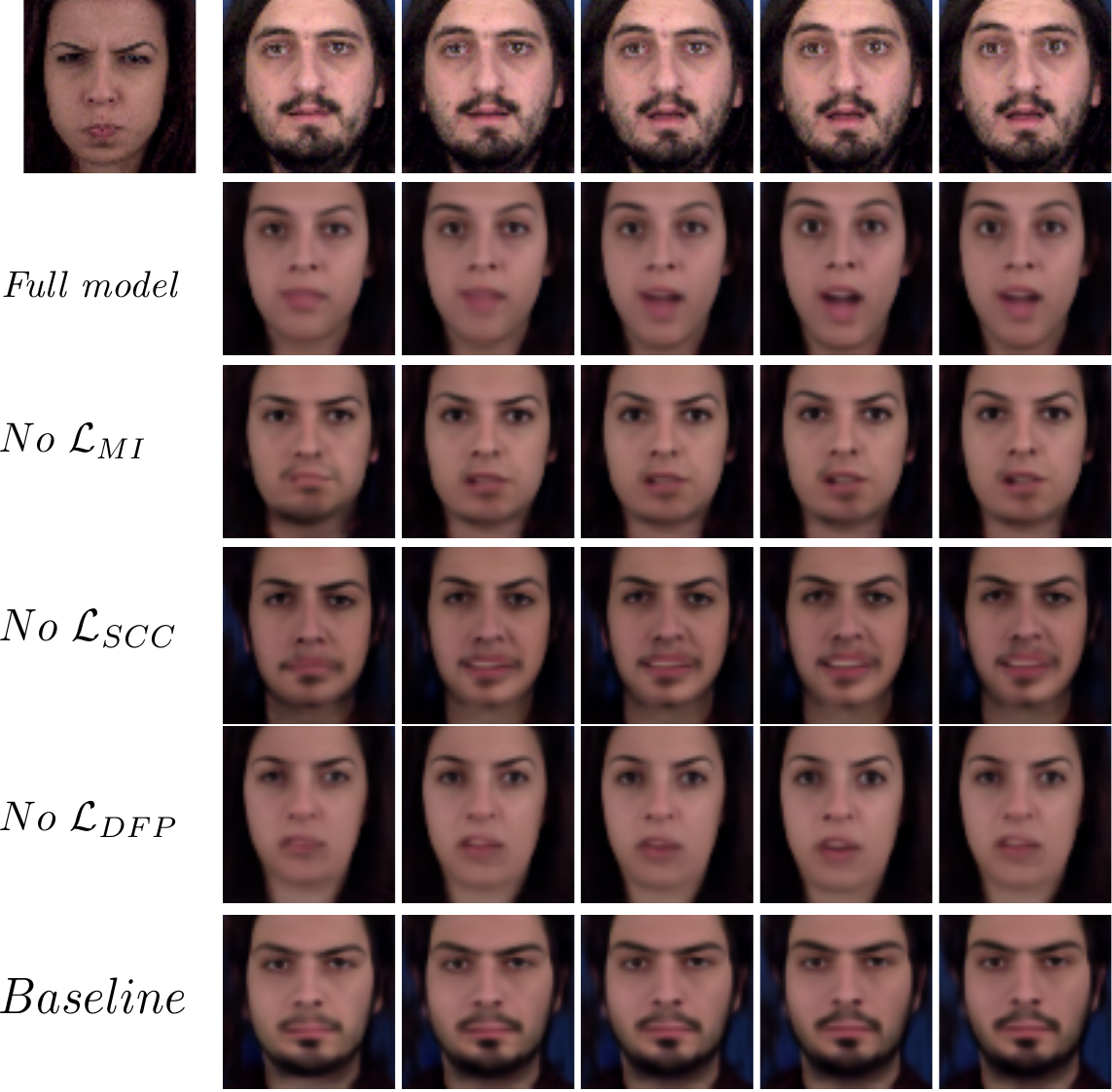}
	\caption{Ablation study on MUG dataset. The first frame of the appearance video and the motion video are shown in the first row.}
	\label{ablation}
	\vspace{-1em}
\end{figure}

\subsection{Experiments on Audio Data}
To demonstrate the general applicability of our model on sequential data, we conduct experiments on audio data, where the time-invariant and time-varying factors are the timbre of a speaker and the linguistic content of a speech, respectively. The dataset we use is TIMIT, which is a corpus of phonemically and lexically transcribed speech of American English speakers of different sexes and dialects~\cite{garofolo1993darpa}, and contains  63000 recordings of read speeches. We split the dataset to training and testing subsets with a ratio of 5:1. As in ~\cite{hsu2017unsupervised}, all the speech are presented as a sequence of 80 dimensional Mel-scale filter bank features.

We quantitatively compare our model with FHVAE and DSVAE on the speaker verification task based on either $\bm{z}_f$ or  $\bm{z}_{1:T}$ , measured by the Equal Error Rate(EER)~\cite{chenafa2008biometric}. 
Note that we expect the speaker can be correctly verified with $\bm{z}_f$ as it encodes the timbre of speakers, and randomly guess with $\bm{z}_{1:T}$ as it ideally only encodes the linguistic content. The results are shown in Table~\ref{tab:timit_verify}. Our model outperforms competing methods in both cases. Especially when based on $\bm{z}_{1:T}$, our model doubles the score of the baseline, indicating our model significantly eliminate the timbre information in $\bm{z}_{1:T}$.

\begin{table}
	\centering
	\caption{Performance comparison on speaker verification. Small errors are better for  $\bm{z}_f$ and large errors are expected for $\bm{z}_{1:T}$.  
	}
	\vspace{-1em}
	\scalebox{0.9}{
	\begin{tabular}{|l|ccc|}
		\hline
		model & feature & dim & EER  \\
		\hline
		FHVAE & $\bm{z}_f$ & 16 & 5.06\% \\
		\hline
		DSVAE & $\bm{z}_f$ & 64 & 4.82\% \\
		& $\bm{z}_{1:T}$ & 64 & 18.89\% \\ \hline 
	    Ours & $\bm{z}_f$ & 64 & \textbf{4.80\%} \\
		& $\bm{z}_{1:T}$ & 64 & \textbf{40.12\%} \\
		\hline
	\end{tabular}
	}
	\label{tab:timit_verify}
	\vspace{-1.5em}
\end{table}

\section{Conclusion}
We propose a self-supervised sequential VAE, which learns disentangled time-invariant and time-varying representations for sequential data. We show that, with readily accessible supervisory signals from data itself and off-the-shelf tools, our model can achieve comparable performance to the fully supervised models that require costly human annotations. The disentangling ability of our model is qualitatively and quantitatively verified on four datasets across video and audio domains. The appealing results on a variety of tasks illustrate that, leveraging self-supervision is a promising direction for representation disentanglement and sequential data generation.  In the future, we plan to extend our model to high-resolution video generation, video prediction and image-to-video generation.

{\small
\bibliographystyle{ieee_fullname}
\bibliography{egbib}

\begin{thebibliography}{10}\itemsep=-1pt

\bibitem{agrawal2015learning}
Pulkit Agrawal, Joao Carreira, and Jitendra Malik.
\newblock Learning to see by moving.
\newblock In {\em Proceedings of the IEEE International Conference on Computer
  Vision}, pages 37--45, 2015.

\bibitem{5617662}
N. {Aifanti}, C. {Papachristou}, and A. {Delopoulos}.
\newblock The mug facial expression database.
\newblock In {\em 11th International Workshop on Image Analysis for Multimedia
  Interactive Services WIAMIS 10}, pages 1--4, April 2010.

\bibitem{arandjelovic2017look}
Relja Arandjelovic and Andrew Zisserman.
\newblock Look, listen and learn.
\newblock In {\em Proceedings of the IEEE International Conference on Computer
  Vision}, pages 609--617, 2017.

\bibitem{bayer2014learning}
Justin Bayer and Christian Osendorfer.
\newblock Learning stochastic recurrent networks.
\newblock {\em arXiv preprint arXiv:1411.7610}, 2014.

\bibitem{bengio2013representation}
Yoshua Bengio, Aaron Courville, and Pascal Vincent.
\newblock Representation learning: A review and new perspectives.
\newblock {\em IEEE transactions on pattern analysis and machine intelligence},
  35(8):1798--1828, 2013.

\bibitem{chen2018isolating}
Tian~Qi Chen, Xuechen Li, Roger~B Grosse, and David~K Duvenaud.
\newblock Isolating sources of disentanglement in variational autoencoders.
\newblock In {\em Advances in Neural Information Processing Systems}, pages
  2610--2620, 2018.

\bibitem{chen2016infogan}
Xi Chen, Yan Duan, Rein Houthooft, John Schulman, Ilya Sutskever, and Pieter
  Abbeel.
\newblock Infogan: Interpretable representation learning by information
  maximizing generative adversarial nets.
\newblock In {\em Advances in neural information processing systems}, pages
  2172--2180, 2016.

\bibitem{chenafa2008biometric}
Mohamed Chenafa, Dan Istrate, Valeriu Vrabie, and Michel Herbin.
\newblock Biometric system based on voice recognition using multiclassifiers.
\newblock In {\em European Workshop on Biometrics and Identity Management},
  pages 206--215. Springer, 2008.

\bibitem{cho2014learning}
Kyunghyun Cho, Bart Van~Merri{\"e}nboer, Caglar Gulcehre, Dzmitry Bahdanau,
  Fethi Bougares, Holger Schwenk, and Yoshua Bengio.
\newblock Learning phrase representations using rnn encoder-decoder for
  statistical machine translation.
\newblock {\em arXiv preprint arXiv:1406.1078}, 2014.

\bibitem{chung2015recurrent}
Junyoung Chung, Kyle Kastner, Laurent Dinh, Kratarth Goel, Aaron~C Courville,
  and Yoshua Bengio.
\newblock A recurrent latent variable model for sequential data.
\newblock In {\em Advances in neural information processing systems}, pages
  2980--2988, 2015.

\bibitem{denton2018stochastic}
Emily Denton and Rob Fergus.
\newblock Stochastic video generation with a learned prior.
\newblock In {\em 35th International Conference on Machine Learning (ICML)},
  2018.

\bibitem{denton2017unsupervised}
Emily~L Denton et~al.
\newblock Unsupervised learning of disentangled representations from video.
\newblock In {\em Advances in neural information processing systems}, pages
  4414--4423, 2017.

\bibitem{doersch2015unsupervised}
Carl Doersch, Abhinav Gupta, and Alexei~A Efros.
\newblock Unsupervised visual representation learning by context prediction.
\newblock In {\em Proceedings of the IEEE International Conference on Computer
  Vision}, pages 1422--1430, 2015.

\bibitem{dong2018style}
Xuanyi Dong, Yan Yan, Wanli Ouyang, and Yi Yang.
\newblock Style aggregated network for facial landmark detection.
\newblock In {\em Proceedings of the IEEE Conference on Computer Vision and
  Pattern Recognition}, pages 379--388, 2018.

\bibitem{esmaeili2018structured}
Babak Esmaeili, Hao Wu, Sarthak Jain, Alican Bozkurt, Narayanaswamy Siddharth,
  Brooks Paige, Dana~H Brooks, Jennifer Dy, and Jan-Willem van~de Meent.
\newblock Structured disentangled representations.
\newblock {\em arXiv preprint arXiv:1804.02086}, 2018.

\bibitem{esmaeili2018hierarchical}
Babak Esmaeili, Hao Wu, Sarthak Jain, N Siddharth, Brooks Paige, and Jan-Willem
  Van~de Meent.
\newblock Hierarchical disentangled representations.
\newblock {\em stat}, 1050:12, 2018.

\bibitem{fabius2014variational}
Otto Fabius and Joost~R van Amersfoort.
\newblock Variational recurrent auto-encoders.
\newblock {\em arXiv preprint arXiv:1412.6581}, 2014.

\bibitem{fan2018SOC}
Deng-Ping Fan, Ming-Ming Cheng, Jiang-Jiang Liu, Shang-Hua Gao, Qibin Hou, and
  Ali Borji.
\newblock Salient objects in clutter: Bringing salient object detection to the
  foreground.
\newblock In {\em European Conference on Computer Vision (ECCV)}. Springer,
  2018.

\bibitem{Fan2019VideoSal}
Deng-Ping Fan, Wenguan Wang, Ming-Ming Cheng, and Jianbing Shen.
\newblock Shifting more attention to video salient object detection.
\newblock In {\em IEEE CVPR}, 2019.

\bibitem{gan2018geometry}
Chuang Gan, Boqing Gong, Kun Liu, Hao Su, and Leonidas~J Guibas.
\newblock Geometry guided convolutional neural networks for self-supervised
  video representation learning.
\newblock In {\em Proceedings of the IEEE Conference on Computer Vision and
  Pattern Recognition}, pages 5589--5597, 2018.

\bibitem{garofolo1993darpa}
John~S Garofolo, L~F Lamel, W~M Fisher, Jonathan~G Fiscus, D~S Pallett, and
  Nancy~L Dahlgren.
\newblock Darpa timit acoustic-phonetic continuous speech corpus cd-rom
  $\{$TIMIT$\}$.
\newblock Technical report, 1993.

\bibitem{goodfellow2014generative}
Ian Goodfellow, Jean Pouget-Abadie, Mehdi Mirza, Bing Xu, David Warde-Farley,
  Sherjil Ozair, Aaron Courville, and Yoshua Bengio.
\newblock Generative adversarial nets.
\newblock In {\em Advances in neural information processing systems}, pages
  2672--2680, 2014.

\bibitem{he2018probabilistic}
Jiawei He, Andreas Lehrmann, Joseph Marino, Greg Mori, and Leonid Sigal.
\newblock Probabilistic video generation using holistic attribute control.
\newblock In {\em Proceedings of the European Conference on Computer Vision
  (ECCV)}, pages 452--467, 2018.

\bibitem{Higgins2017betaVAELB}
Irina Higgins, Lo{\"i}c Matthey, Arka Pal, Christopher Burgess, Xavier Glorot,
  Matthew~M Botvinick, Shakir Mohamed, and Alexander Lerchner.
\newblock beta-vae: Learning basic visual concepts with a constrained
  variational framework.
\newblock In {\em ICLR}, 2017.

\bibitem{hinton1994autoencoders}
Geoffrey~E Hinton and Richard~S Zemel.
\newblock Autoencoders, minimum description length and helmholtz free energy.
\newblock In {\em Advances in neural information processing systems}, pages
  3--10, 1994.

\bibitem{hochreiter1997long}
Sepp Hochreiter and J{\"u}rgen Schmidhuber.
\newblock Long short-term memory.
\newblock {\em Neural computation}, 9(8):1735--1780, 1997.

\bibitem{hsu2017unsupervised}
Wei-Ning Hsu, Yu Zhang, and James Glass.
\newblock Unsupervised learning of disentangled and interpretable
  representations from sequential data.
\newblock In {\em Advances in neural information processing systems}, pages
  1878--1889, 2017.

\bibitem{ilg2017flownet}
Eddy Ilg, Nikolaus Mayer, Tonmoy Saikia, Margret Keuper, Alexey Dosovitskiy,
  and Thomas Brox.
\newblock Flownet 2.0: Evolution of optical flow estimation with deep networks.
\newblock In {\em Proceedings of the IEEE conference on computer vision and
  pattern recognition}, pages 2462--2470, 2017.

\bibitem{jayaraman2015learning}
Dinesh Jayaraman and Kristen Grauman.
\newblock Learning image representations tied to ego-motion.
\newblock In {\em Proceedings of the IEEE International Conference on Computer
  Vision}, pages 1413--1421, 2015.

\bibitem{jeon2018ib}
Insu Jeon, Wonkwang Lee, and Gunhee Kim.
\newblock Ib-gan: Disentangled representation learning with information
  bottleneck gan.
\newblock 2018.

\bibitem{jing2019self}
Longlong Jing and Yingli Tian.
\newblock Self-supervised visual feature learning with deep neural networks: A
  survey.
\newblock {\em arXiv preprint arXiv:1902.06162}, 2019.

\bibitem{kay2017kinetics}
Will Kay, Joao Carreira, Karen Simonyan, Brian Zhang, Chloe Hillier, Sudheendra
  Vijayanarasimhan, Fabio Viola, Tim Green, Trevor Back, Paul Natsev, et~al.
\newblock The kinetics human action video dataset.
\newblock {\em arXiv preprint arXiv:1705.06950}, 2017.

\bibitem{kim2018disentangling}
Hyunjik Kim and Andriy Mnih.
\newblock Disentangling by factorising.
\newblock {\em arXiv preprint arXiv:1802.05983}, 2018.

\bibitem{kingma2014adam}
Diederik~P Kingma and Jimmy Ba.
\newblock Adam: A method for stochastic optimization.
\newblock {\em arXiv preprint arXiv:1412.6980}, 2014.

\bibitem{kingma2013auto}
Diederik~P Kingma and Max Welling.
\newblock Auto-encoding variational bayes.
\newblock {\em arXiv preprint arXiv:1312.6114}, 2013.

\bibitem{larsson2017colorization}
Gustav Larsson, Michael Maire, and Gregory Shakhnarovich.
\newblock Colorization as a proxy task for visual understanding.
\newblock In {\em Proceedings of the IEEE Conference on Computer Vision and
  Pattern Recognition}, pages 6874--6883, 2017.

\bibitem{li2018disentangled}
Yingzhen Li and Stephan Mandt.
\newblock Disentangled sequential autoencoder.
\newblock In {\em International Conference on Machine Learning (ICML)}, 2018.

\bibitem{locatello2018challenging}
Francesco Locatello, Stefan Bauer, Mario Lucic, Sylvain Gelly, Bernhard
  Sch{\"o}lkopf, and Olivier Bachem.
\newblock Challenging common assumptions in the unsupervised learning of
  disentangled representations.
\newblock In {\em ICML}, 2019.

\bibitem{misra2016shuffle}
Ishan Misra, C~Lawrence Zitnick, and Martial Hebert.
\newblock Shuffle and learn: unsupervised learning using temporal order
  verification.
\newblock In {\em European Conference on Computer Vision}, pages 527--544.
  Springer, 2016.

\bibitem{noh2015learning}
Hyeonwoo Noh, Seunghoon Hong, and Bohyung Han.
\newblock Learning deconvolution network for semantic segmentation.
\newblock In {\em Proceedings of the IEEE international conference on computer
  vision}, pages 1520--1528, 2015.

\bibitem{noroozi2016unsupervised}
Mehdi Noroozi and Paolo Favaro.
\newblock Unsupervised learning of visual representations by solving jigsaw
  puzzles.
\newblock In {\em European Conference on Computer Vision}, pages 69--84.
  Springer, 2016.

\bibitem{owens2016ambient}
Andrew Owens, Jiajun Wu, Josh~H McDermott, William~T Freeman, and Antonio
  Torralba.
\newblock Ambient sound provides supervision for visual learning.
\newblock In {\em European conference on computer vision}, pages 801--816.
  Springer, 2016.

\bibitem{paszke2017automatic}
Adam Paszke, Sam Gross, Soumith Chintala, Gregory Chanan, Edward Yang, Zachary
  DeVito, Zeming Lin, Alban Desmaison, Luca Antiga, and Adam Lerer.
\newblock Automatic differentiation in pytorch.
\newblock 2017.

\bibitem{pathak2017learning}
Deepak Pathak, Ross Girshick, Piotr Doll{\'a}r, Trevor Darrell, and Bharath
  Hariharan.
\newblock Learning features by watching objects move.
\newblock In {\em Proceedings of the IEEE Conference on Computer Vision and
  Pattern Recognition}, pages 2701--2710, 2017.

\bibitem{ren2015faster}
Shaoqing Ren, Kaiming He, Ross Girshick, and Jian Sun.
\newblock Faster r-cnn: Towards real-time object detection with region proposal
  networks.
\newblock In {\em Advances in neural information processing systems}, pages
  91--99, 2015.

\bibitem{salimans2016improved}
Tim Salimans, Ian Goodfellow, Wojciech Zaremba, Vicki Cheung, Alec Radford, and
  Xi Chen.
\newblock Improved techniques for training gans.
\newblock In {\em Advances in neural information processing systems}, pages
  2234--2242, 2016.

\bibitem{Siarohin_2019_CVPR}
Aliaksandr Siarohin, Stéphane Lathuilière, Sergey Tulyakov, Elisa Ricci, and
  Nicu Sebe.
\newblock Animating arbitrary objects via deep motion transfer.
\newblock In {\em The IEEE Conference on Computer Vision and Pattern
  Recognition (CVPR)}, June 2019.

\bibitem{sun2018two}
Ximeng Sun, Huijuan Xu, and Kate Saenko.
\newblock A two-stream variational adversarial network for video generation.
\newblock {\em arXiv preprint arXiv:1812.01037}, 2018.

\bibitem{tulyakov2018mocogan}
Sergey Tulyakov, Ming-Yu Liu, Xiaodong Yang, and Jan Kautz.
\newblock Mocogan: Decomposing motion and content for video generation.
\newblock In {\em Proceedings of the IEEE conference on computer vision and
  pattern recognition}, pages 1526--1535, 2018.

\bibitem{villegas2017decomposing}
Ruben Villegas, Jimei Yang, Seunghoon Hong, Xunyu Lin, and Honglak Lee.
\newblock Decomposing motion and content for natural video sequence prediction.
\newblock {\em arXiv preprint arXiv:1706.08033}, 2017.

\bibitem{vincent2008extracting}
Pascal Vincent, Hugo Larochelle, Yoshua Bengio, and Pierre-Antoine Manzagol.
\newblock Extracting and composing robust features with denoising autoencoders.
\newblock In {\em Proceedings of the 25th international conference on Machine
  learning}, pages 1096--1103. ACM, 2008.

\bibitem{Wang_2019_CVPR}
Jiangliu Wang, Jianbo Jiao, Linchao Bao, Shengfeng He, Yunhui Liu, and Wei Liu.
\newblock Self-supervised spatio-temporal representation learning for videos by
  predicting motion and appearance statistics.
\newblock In {\em The IEEE Conference on Computer Vision and Pattern
  Recognition (CVPR)}, June 2019.

\bibitem{wang2015unsupervised}
Xiaolong Wang and Abhinav Gupta.
\newblock Unsupervised learning of visual representations using videos.
\newblock In {\em Proceedings of the IEEE International Conference on Computer
  Vision}, pages 2794--2802, 2015.

\bibitem{zhang2016colorful}
Richard Zhang, Phillip Isola, and Alexei~A Efros.
\newblock Colorful image colorization.
\newblock In {\em European conference on computer vision}, pages 649--666.
  Springer, 2016.

\bibitem{zhu2018generative}
Yizhe Zhu, Mohamed Elhoseiny, Bingchen Liu, Xi Peng, and Ahmed Elgammal.
\newblock A generative adversarial approach for zero-shot learning from noisy
  texts.
\newblock In {\em Proceedings of the IEEE Conference on Computer Vision and
  Pattern Recognition (CVPR)}, pages 1004--1013, 2018.

\bibitem{zhu2019learning}
Yizhe Zhu, Jianwen Xie, Bingchen Liu, and Ahmed Elgammal.
\newblock Learning feature-to-feature translator by alternating
  back-propagation for generative zero-shot learning.
\newblock In {\em Proceedings of the IEEE International Conference on Computer
  Vision (ICCV)}, Oct 2019.

\bibitem{zhu2019semantic}
Yizhe Zhu, Jianwen Xie, Zhiqiang Tang, Xi Peng, and Ahmed Elgammal.
\newblock Semantic-guided multi-attention localization for zero-shot learning.
\newblock In {\em Thirty-third Conference on Neural Information Processing
  Systems (NeurIPS)}, Dec 2019.

\end{thebibliography}
}

\begin{appendices}
	\section{ Minibatch Weighted Sampling}
	Minibatch Weighted Sampling is an estimator of the posterior $q(z)$ introduced by~\cite{chen2018isolating}.  Let $N$ be the size of a dataset and $M$ be the size of a minibatch, the entropy of the posterior distribution can be estimated based on a minibatch:
	\begingroup\makeatletter\def\f@size{9}\check@mathfonts
	\def\maketag@@@#1{\hbox{\m@th\large\normalfont#1}}
	\begin{equation}
	\begin{split}
	\E_{q(z)}[\log q(z)] \approx \frac{1}{M}\sum_{i=1}^M \left[ \log \sum_{j=1}^{M} q(z(x_i)|x_j) - \log(NM) \right]\\
	\end{split}
	\end{equation}\endgroup
	The readers can refer to ~\cite{chen2018isolating} for the details. 
	
	In our model, the posterior of the latent variable $z$ can be factorized as $q(z|x) = q(z_f|x)q(z_t|x)$. Thus the entropy of the joint distribution can be estimated:　
	\begingroup\makeatletter\def\f@size{9}\check@mathfonts
	\def\maketag@@@#1{\hbox{\m@th\large\normalfont#1}}
	\begin{equation}
	\begin{split}
	&\E_{q(z)}[\log q(z)] \approx \\
	&\frac{1}{M}\sum_{i=1}^M \left[ \log \sum_{j=1}^{M} q(z_f(x_i)|x_j)q(z_t(x_i)|x_j) - \log(NM) \right]
	\end{split}
	\end{equation}\endgroup

	\begin{lemma}
		\label{pythagorean}
		Given a dataset of N samples $\D_N = \{x_1, ..., x_N\}$ with a distribution $p(x)$ and  a minibatch of $M$ samples  $\B_M = \{x_1, ..., x_M\}$ drawn i.i.d. from $p(x)$, and assume the posterior of the latent variable $z$  can be factorized as: $q(z|x) = q(z_1|x)q(z_2|x)$, the lower bound of $\E_{q(z)}[\log q(z_n)]$, $n=1$ or $2$, is : 
		\begingroup\makeatletter\def\f@size{9}\check@mathfonts
		\def\maketag@@@#1{\hbox{\m@th\large\normalfont#1}}
		\begin{equation*}
		\E_{q(z,x)} \left[ \log \E_{r(\B_M|x)} \left[ \frac{1}{NM} \sum_{m=1}^{M} q(z_n|x_m) \right] \right],  
		\end{equation*}\endgroup
		where $r(\B_M|x)$ denotes the probability of a sampled minibatch where one of the elements is fixed to be $x$ and the rest are sampled i.i.d. from $p(x)$.
	\end{lemma}
	\begin{proof}
		For any sampled batch instance $\B_M$, $p(\B_M) = \left(\nicefrac{1}{N}\right)^M$, and when one of the elements is fixed to be $x$,  $r(\B_M|x) = \left(\nicefrac{1}{N}\right)^{M-1}$.
		\begin{equation*}
		\begin{split}
		&\E_{q(z)} \left[ \log q(z_n) \right] \\
		=&\E_{q(z,x)} \left[ \log \E_{x' \sim p(x)} \left[ q(z_n|x') \right] \right] \\
		=&\E_{q(z,x)} \left[ \log \E_{p(\B_M)} \left[ \frac{1}{M} \sum_{m=1}^{M} q(z_n|x_m) \right] \right] \\
		\geq&\E_{q(z,x)} \left[ \log \E_{r(\B_M|x)} \left[ \frac{p(\B_M)}{r(\B_M|x)}\frac{1}{M} \sum_{m=1}^{M} q(z_n|x_m) \right] \right] \\
		=&\E_{q(z,x)} \left[ \log \E_{r(\B_M|x)} \left[ \frac{1}{NM} \sum_{m=1}^{M} q(z_n|x_m) \right] \right] \\
		\end{split}
		\end{equation*}
		The inequality is due to $r$ having a support that is a subset of that of $p$. 
	\end{proof}
	Following Lemma~\ref{pythagorean}, when provided with a minibatch of samples $\{x_1, ..., n_M\}$, we can use estimate the lower bound as:
	\begingroup\makeatletter\def\f@size{9}\check@mathfonts
	\def\maketag@@@#1{\hbox{\m@th\large\normalfont#1}}
	\begin{equation}
	\begin{split}
	\E_{q(z)}[\log q(z_f)] \approx \frac{1}{M}\sum_{i=1}^M \left[ \log \sum_{j=1}^{M} q(z_f(x_i)|x_j) - \log(NM) \right] \\
	\E_{q(z)}[\log q(z_t)] \approx \frac{1}{M}\sum_{i=1}^M \left[ \log \sum_{j=1}^{M} q(z_t(x_i)|x_j) - \log(NM) \right]
	\end{split}
	\end{equation}\endgroup
	where $z_f(x_i)$ is a sample from $q(z_f|x_i)$, and $z_t(x_i)$ is a sample from $q(z_t|x_i)$.
	
	\section{Derivation of Objective Function}
	We show the derivation of objective function in Eq.5. The observe model is defined as: 
	\begin{equation}
	\begin{split}
	p(\bm{x}_{1:T}) &= \int p(\bm{x}_{1:T}, \bm{z}) d\bm{z} \\
	&= \int\int p(\bm{x}_{1:T}, \bm{z}_f, \bm{z}_{1:T}) d \bm{z}_f d \bm{z}_{1:T}
	\end{split}
	\end{equation}
	To avoid the intractable integration over $\bm{z}_f$ and $\bm{z}_{1:T}$, variational inference introduces an  posterior approximation $q(\bm{z}_f, \bm{z}_{1:T}| \bm{x}_{1:T})$.  A variational lower bound of log $p(\bm{x}_{1:T})$ is: 
	\clearpage
	\begin{equation}
	\begin{split}
	\mathcal{L} = &\E_{q(\bm{z}_f, \bm{z}_{1:T}| \bm{x}_{1:T})}\left[ \log \frac{p(\bm{z}_f, \bm{z}_{1:T}, \bm{x}_{1:T})}{q(\bm{z}_f, \bm{z}_{1:T}| \bm{x}_{1:T})}\right]\\
	=&\E_{q(\bm{z}_f, \bm{z}_{1:T}| \bm{x}_{1:T})}\left[\log\frac{p(\bm{z}_f) \prod^T_{t=1} p(\bm{x}_t| \bm{z}_f, \bm{z}_t) p(\bm{z}_t|\bm{z}_{<t})}{q(\bm{z}_f|\bm{x}_{1:T})\prod^T_{t=1} q(\bm{z}_t|\bm{x}_{\leq t})} \right]\\
	=&\E_{q(\bm{z}_f, \bm{z}_{1:T}| \bm{x}_{1:T})}\left[ \sum ^T_{t=1} \log p(\bm{x}_t| \bm{z}_f, \bm{z}_t)\right]\\
	& - \E_{q(\bm{z}_f, \bm{z}_{1:T}| \bm{x}_{1:T})}\left[\log\frac{q(\bm{z}_f|\bm{x}_{1:T})}{p(\bm{z}_f)} \right]\\
	& - \E_{q(\bm{z}_f, \bm{z}_{1:T}| \bm{x}_{1:T})}\left[ \sum ^T_{t=1} \log \frac{q(\bm{z}_t|\bm{x}_{\leq t})}{p(\bm{z}_t|\bm{z}_{<t})} \right]\\
	= &\E_{q(\bm{z}_f, \bm{z}_{1:T}| \bm{x}_{1:T})}\left[ \sum ^T_{t=1} \log p(\bm{x}_t| \bm{z}_f, \bm{z}_t)\right]\\
	& - \E_{q(\bm{z}_f| \bm{x}_{1:T})}\left[\log\frac{q(\bm{z}_f|\bm{x}_{1:T})}{p(\bm{z}_f)} \right]\\
	& - \sum ^T_{t=1} \E_{q(\bm{z}_t| \bm{x}_{\leq T})}\left[  \log \frac{q(\bm{z}_t|\bm{x}_{\leq t})}{p(\bm{z}_t|\bm{z}_{<t})} \right]\\
	=&\E_{q(\bm{z}_f, \bm{z}_{1:T}| \bm{x}_{1:T})}\left[ \sum ^T_{t=1} \log p(\bm{x}_t| \bm{z}_f, \bm{z}_t)\right] \\&- \KL(q(\bm{z}_f|\bm{x}_{1:T})|| p(\bm{z}_f))\\& - \sum ^T_{t=1} \KL( q(\bm{z}_t|\bm{x}_{\leq t}) || p(\bm{z}_t|\bm{z}_{<t}))
	\end{split}
	\end{equation}
	we get line 2 from line 1 by plugging the Eq.2 and Eq.4.

	\section{Metrics Definition}
	Similar to~\cite{sun2018two}, we introduce the definitions of classification accuracy, Inception Score, Inter-Entropy, Intra-Entropy when groundtruth labels can not directly be provided to the model. Let $x$ be the generated video based on the representation of a real video $x_{real}$ with the label $y$ or directly conditioned on the label $y$ (e.g., MoCoGAN). 
	We have a classifier that is pretrained to predict the labels of real videos. 
	\begin{itemize}
		\item \textbf{Classification Accuracy} ($Acc$) measures the percentage of the agreement of predicted labels between the generated video $x$ and the given real video $x_{real}$. Higher classification accuracy indicates that  the generated video is more recognizable and  the corresponding representation is better disentangled from other representation.

		\item \textbf{Inception Score} $IS$ measures the KL divergence between the conditional label distribution $p(y|x)$ and the marginal distribution $p(y)$. 
		
		\begin{equation}
		IS = \exp(\E_{p(x)}[ KL(p(y|x) || p(y))])
		\end{equation}
		
		\item \textbf{Inter-Entropy} $H(y)$ is the entropy of the marginal distribution $p(y)$:
		\begin{equation}
		H(y) = - \sum_y p(y) \log p(y),
		\end{equation}
		where $p(y) = \frac{1}{N} \sum^N_{i=1} p(y|x)$. Higher $H(y)$ means the model generates more diverse results. 
		
		\item \textbf{Intra-Entropy} $H(y|x)$ is entropy of the conditional class distribution $p(y|x)$. 
		\begin{equation}
		H(y|x) = - \sum_y p(y|x) \log p(y|x), 
		\end{equation}
		
		Lower $H(y|x)$ indicates the generated video is more realistic. 
	\end{itemize}
	
	For the speaker verification task in the audio dataset TIMIT, we use the metric \textbf{Equal Error Rate(EER)}. The threshold value is tuned to make the false acceptance rate equal to the false rejection rate. The common value is referred to as the Equal Error Rate.

	\section{Details on Architecture and Training} 
	We implement our model using PyTorch~\cite{paszke2017automatic} and use the Adam optimizer~\cite{kingma2014adam} with $\beta_1 = 0.9$ and $\beta_1 = 0.999$. The learning rate is set to $10^{-3}$ and the batch size is set to 16. Our model is trained with 1000 epochs for each dataset on a GTX 1080 Ti GPU. 
	
	The detailed architecture description of the encoder and decoder of our S3VAE is summarized in Table~\ref{table:arch}. The visual feature of 128d from the frame encoder is fed into an LSTM with one hidden layer (256d) and the LSTM outputs the parameters $\mu$ and $\sigma$ for the Gaussian multivariate distribution of the static variable $\bm{z}_f$ of dimension $d_{zf}$.  The output of the LSTM is fed into another LSTM with one hidden layer (256d) to produce the parameters $\mu$ and $\sigma$ for the Gaussian multivariate distribution of the dynamic variable $\bm{z}_t$ of dimension $d_{zt}$ for each time step.
	Besides, we adopt a trainable LSTM to parameterize the prior of the dynamic variable, $\phi_R^{prior}$. 
	
	The dimensionality of latent variables $(d_{zf}, d_{zt})$ is set to $(256,32), (256,32), (8,128)$ for SMMNIST, Sprite, MUG, respectively. The balancing parameters $\lambda_1$   $\lambda_2$ and $\lambda_3$ are set to 1000, 100, 1, respectively, for all datasets.

	\begin{table}
		\caption{Frame Encoder and decoder of S3VAE for SMMNIST, Sprite, MUG datasets. Let sd denote stride, pd, padding; ch, channel; lReLU, leakyReLU. $d_{img}$ is the number of image channels, which is 1 for SMMNIST and 3 for Sprite and MUG datasets. }
		
		\label{table:arch}
		\centering
		\scalebox{0.9}{
			\begin{tabular}{ll}
				\hline 
				Encoder        &   Decoder \\ \hline
				Input 64x64 RGB image   & Input z       \\
				4x4 conv(sd 2, pd 1, ch 64)   & 4x4 convTrans(sd 1, pd 0, ch 512) \\ 
				BN, lReLU(0.2),               & BN, ReLU, upsample \\
				4x4 conv(sd 2, pd 1, ch 128)  & 3x3 conv(sd 1, pd 1, ch 256)      \\ 
				BN, lReLU(0.2),               & BN, ReLU, upsample  \\
				4x4 conv(sd 2, pd 1, ch 256)  & 3x3 conv(sd 1, pd 1, ch 128)   \\
				BN, lReLU(0.2),               & BN, ReLU, upsample \\
				4x4 conv(sd 2, pd 1, ch 512)  & 3x3 conv(sd 1, pd 1, ch 128) \\
				BN, lReLU(0.2),               & BN, ReLU, upsample \\ 
				4x4 conv(sd 1, pd 0, ch 128)  & 3x3 conv(sd 1, pd 1, ch 64) \\
				BN, Tanh                      & BN, ReLU \\
				& 1x1 conv(sd 1, pd 0, ch $d_{img}$)\\ &sigmoid \\
				\hline 
			\end{tabular}
		}
	\end{table}

	\section{Representation swapping on audio data}
	
	We now show the qualitative results of representation swapping. 
	In Figure~\ref{audio_swapping}, we show the results of representation swapping.  Each heatmap shows the mel-scale filter bank features of 200ms in the frequency domain, where the x-axis is temporal with 20 steps, and the y-axis represents the value of frequency.  As marked in the black rectangle, 24 examples are generated by combining four static variables extracted from the samples in the first column and six dynamic variables extracted from samples in the first row.
	
	As can be observed, in each column, the linguistic phonetic-level contents, reflected by the formants along the temporal axis, are kept almost the same. On the other hand, the timbres are reflected as the harmonics in the heatmap, which correspond to horizontal light stripes. In each row,  the harmonics of heatmaps keep consistent, indicating the timbre of the speaker is preserved.  Overall, the results demonstrate the ability of our model to disentangle the representation of audio data.

	\begin{figure}
		\centering 
		\includegraphics[width=0.48\textwidth]{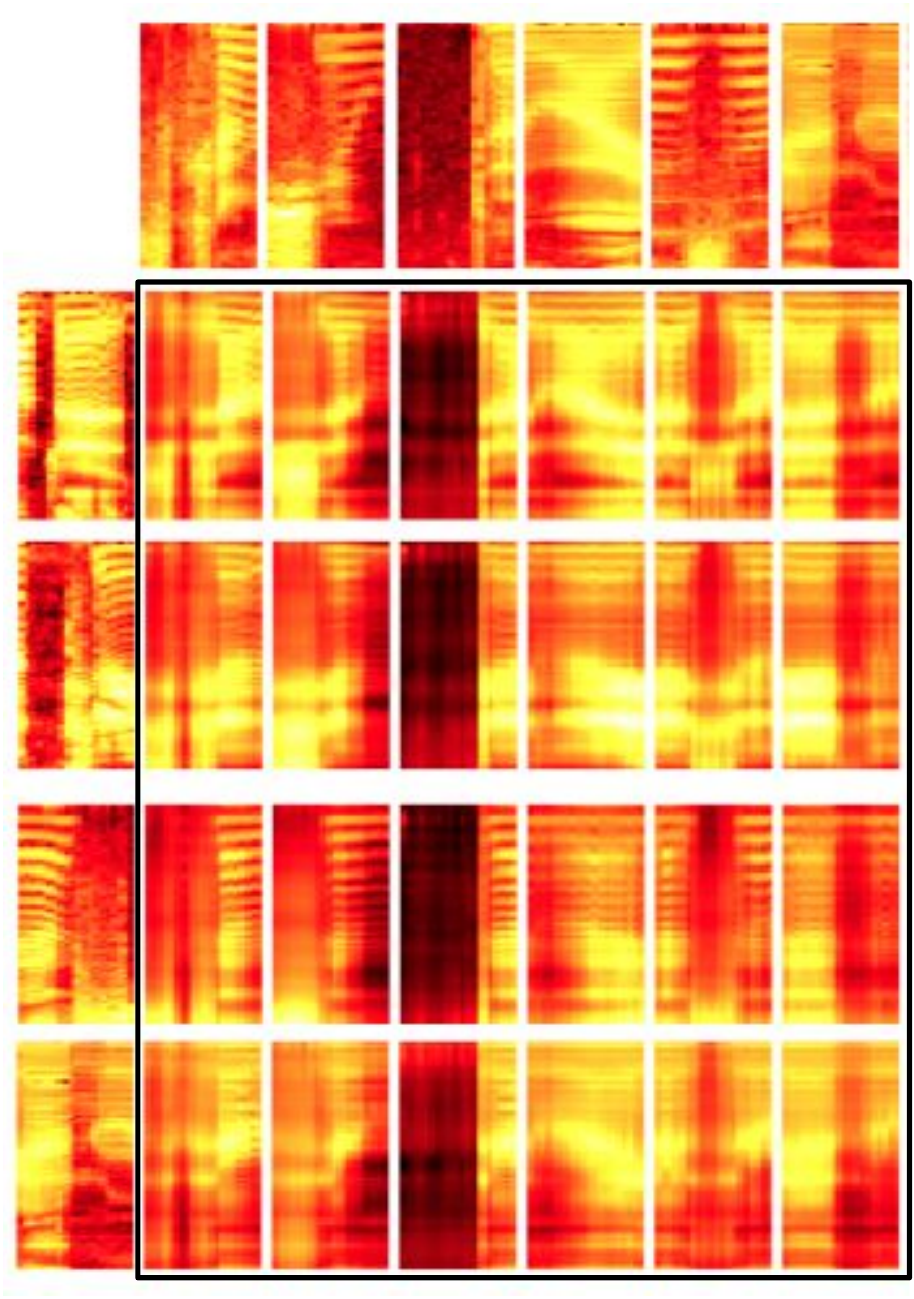}
		\caption{Representation swapping. Each heatmap shows the mel-scale filter bank features of 200ms in the frequency domain, where the x-axis is temporal with 20 steps, and the y-axis reflects the value of frequency. The first row shows the real data where $\bm{z}_{1:T}$, encoding linguistic content,  is extracted while the first column is the real data where $\bm{z}_f$, encoding timbre, is extracted. Each of the rest sequential data is generated based on the axis-corresponding $\bm{z}_f$ and  $\bm{z}_{1:T}$.}
		\label{audio_swapping}
	\end{figure}

	\section{More Qualitative Results}
	We report additional qualitative results on representation swapping in Figure~\ref{s1},~\ref{s2} and ~\ref{s3}. These qualitative results further illustrate the ability of our method to disentangle the static and dynamic representations. As can be seen, the generated videos follow the motion of $\bm{V}_m$ while preserving the appearance of the $\bm{V}_a$. 
	
	Besides, to validate the effectiveness of our method on video generation manipulation, we show qualitative results of video generation with fixed representation. Specifically, the videos first are generated by fixing the static representation $\bm{z}_f$  and sampling the dynamic representation $\bm{z}_{1:T}$. As shown in \ref{m1a},~\ref{m2a} and~\ref{m3a}, the generated videos have the sames appearance but perform various motions. 
	
	Then the videos are generated by fixing the dynamic representation $\bm{z}_{1:T}$ and sampling  the static representation $\bm{z}_f$. As shown in ~\ref{m1b},~\ref{m2b} and ~\ref{m3b}, the generated videos have various appearance but perform the same motions.

	\clearpage
	\begin{figure*}
		\centering 
		\includegraphics[width=0.9\textwidth]{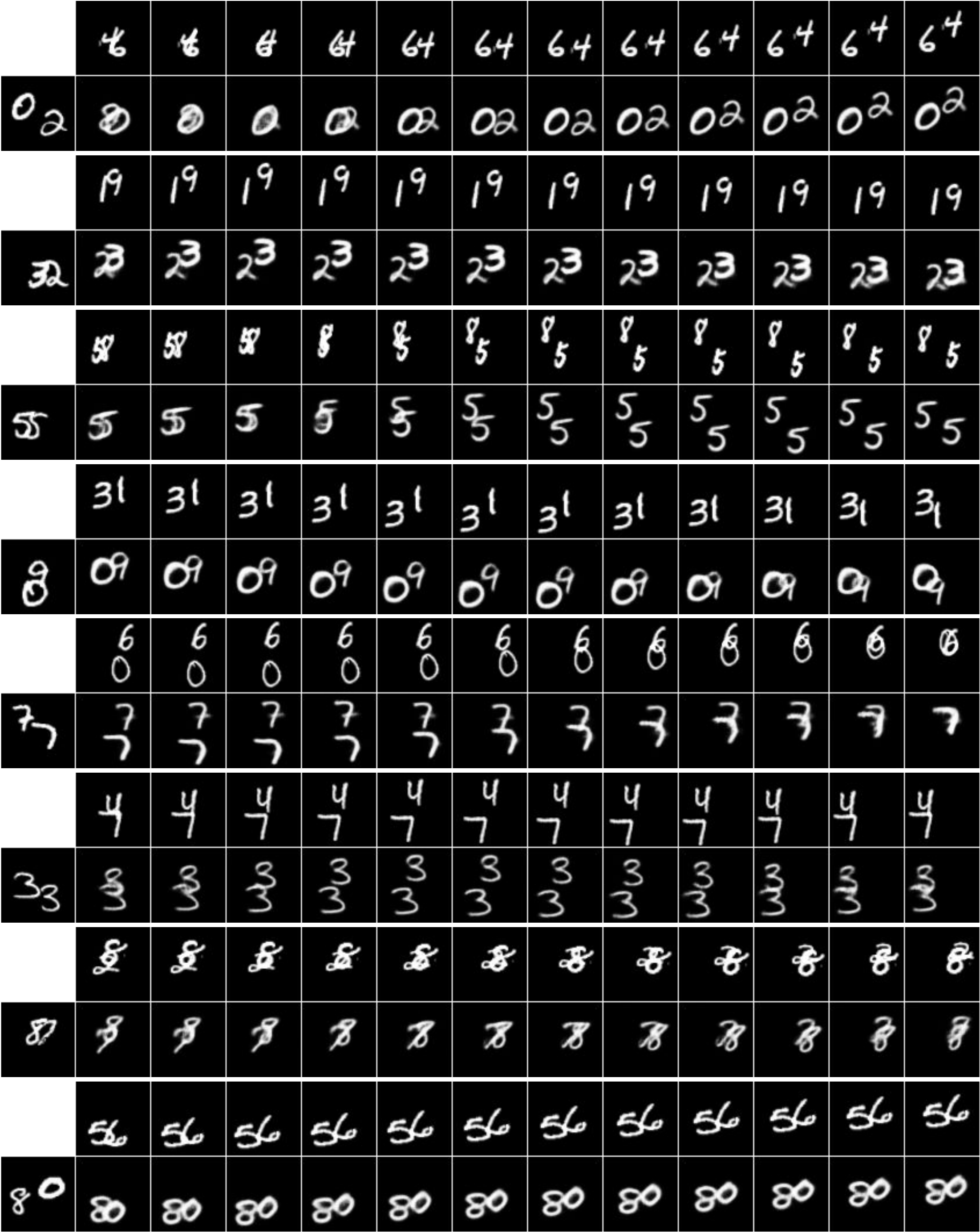}
		\caption{Qualitative results for Representation swapping on SMMNIST dataset. In each panel, the first row is $\bm{V}_m$ that provides the dynamic representation $\bm{z}_{1:T}$ and the first image of the second row is one frame of $\bm{V}_a$ that provides the static representation $\bm{z}_f$. The video generated based on  $\bm{z}_{1:T}$ and $\bm{z}_f$ is shown in the second row. }
		\label{s1}
	\end{figure*}
	\begin{figure*}
		\centering    
		\begin{subfigure}[b]{0.9\textwidth}
			\includegraphics[width=\textwidth]{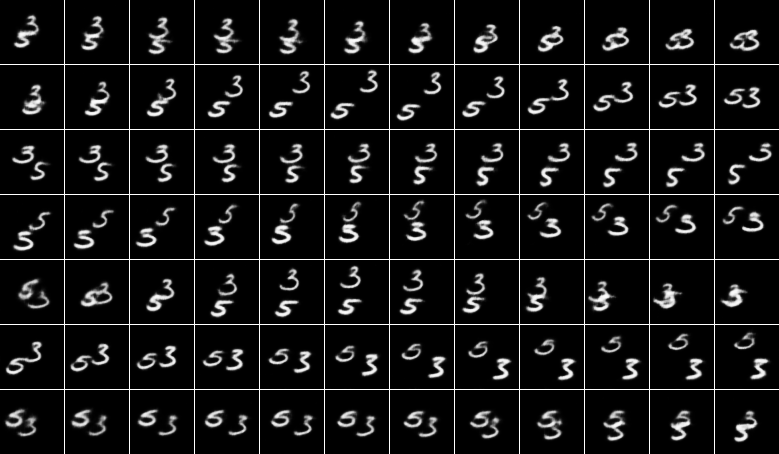}
			\caption{Videos generated by fixing the static representation $\bm{z}_f$  and sampling the dynamic representation $\bm{z}_{1:T}$. Each row shows one generated video sequence. All videos show the same digits, which moves in various directions in different videos.}
			\label{m1a}
		\end{subfigure}	\\
		\begin{subfigure}[b]{0.9\textwidth}
			\includegraphics[width=\textwidth]{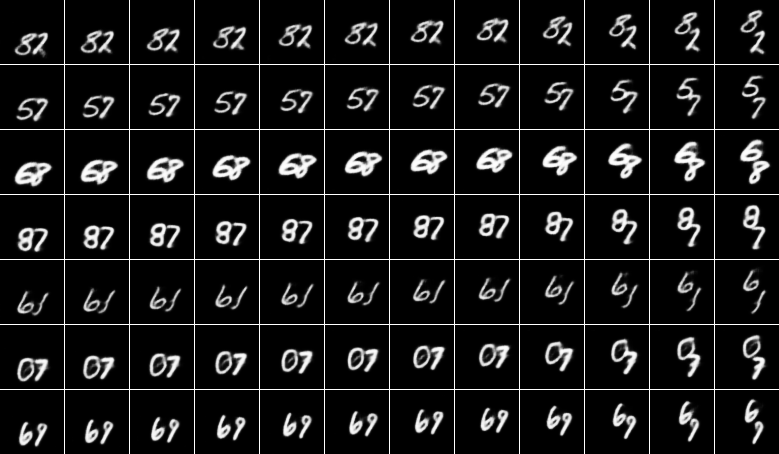}
			\caption{Videos generated by fixing the dynamic representation $\bm{z}_{1:T}$ and sampling the static representation $\bm{z}_f$.  Each row shows one generated video sequence. Different videos show various digits, which perform the same motion.}
			\label{m1b}
		\end{subfigure}
		
		\caption{Manipulating video generation on SMMNIST dataset.}
		\label{m1}
	\end{figure*}
	
	\begin{figure*}
		\centering 
		\includegraphics[width=0.6\textwidth]{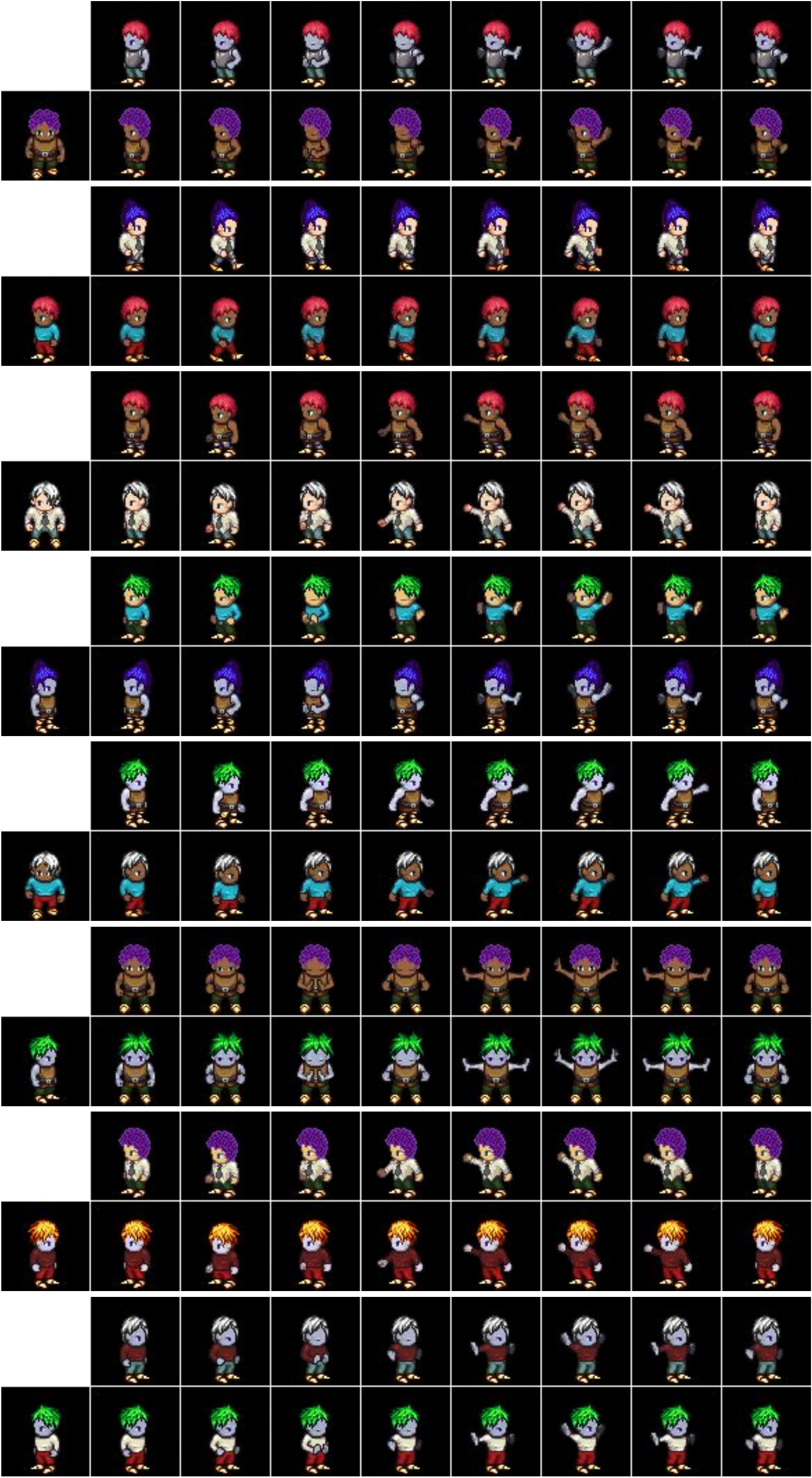}
		\caption{Qualitative results for Representation swapping on Sprite dataset. In each panel, the first row is $\bm{V}_m$ that provides the dynamic representation $\bm{z}_{1:T}$ and the first image of the second row is one frame of $\bm{V}_a$ that provides the static representation $\bm{z}_f$. The video generated based on  $\bm{z}_{1:T}$ and $\bm{z}_f$ is shown in the second row.}
		\label{s2}
	\end{figure*}
	\begin{figure*}
		\centering    
		\begin{subfigure}[b]{0.58\textwidth}
			\includegraphics[width=\textwidth]{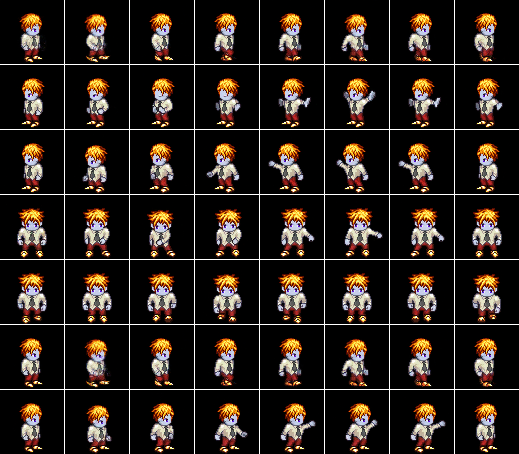}
			\caption{Videos generated by fixing the static representation $\bm{z}_f$  and sampling the dynamic representation $\bm{z}_{1:T}$. Each row shows one generated video sequence. All videos show the same character, which performs various actions in various directions.}
			\label{m2a}
		\end{subfigure}	\\
		\begin{subfigure}[b]{0.58\textwidth}
			\includegraphics[width=\textwidth]{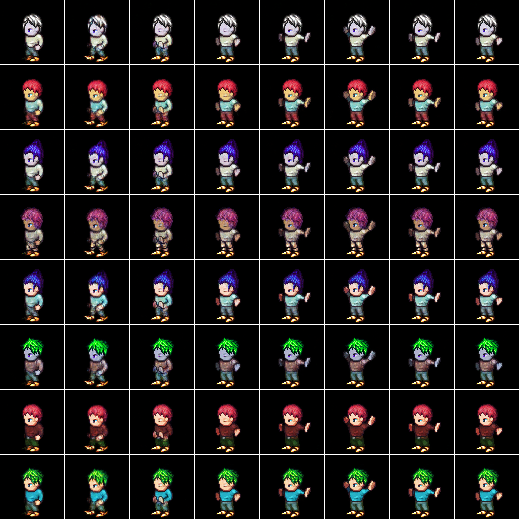}
			\caption{Videos generated by fixing the dynamic representation $\bm{z}_{1:T}$   and sampling the static representation $\bm{z}_f$.  Different videos show various characters, which perform the same motion towards the same direction.}
			\label{m2b}
		\end{subfigure}
		
		\caption{Manipulating video generation on Sprite dataset.}
		\label{m2}
	\end{figure*}

	\begin{figure*}
		\centering 
		\includegraphics[width=0.9\textwidth]{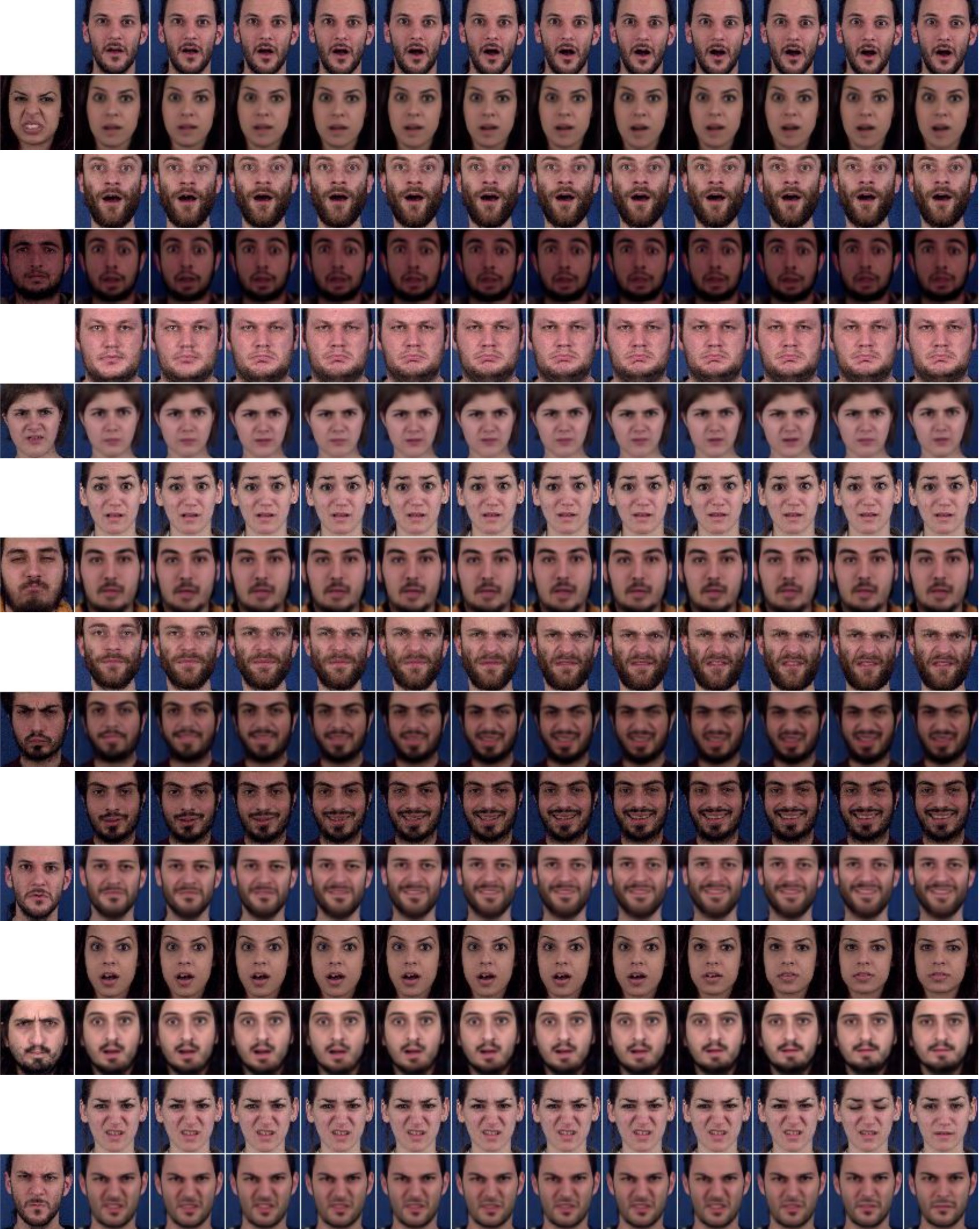}
		\caption{Qualitative results for Representation swapping on MUG dataset. In each panel, the first row is $\bm{V}_m$ that provides the dynamic representation $\bm{z}_{1:T}$ and the first image of the second row is one frame of $\bm{V}_a$ that provides the static representation $\bm{z}_f$. The video generated based on  $\bm{z}_{1:T}$ and $\bm{z}_f$ is shown in the second row.}
		\label{s3}
	\end{figure*}
	
	\begin{figure*}
		\centering    
		\begin{subfigure}[b]{0.9\textwidth}
			\includegraphics[width=\textwidth]{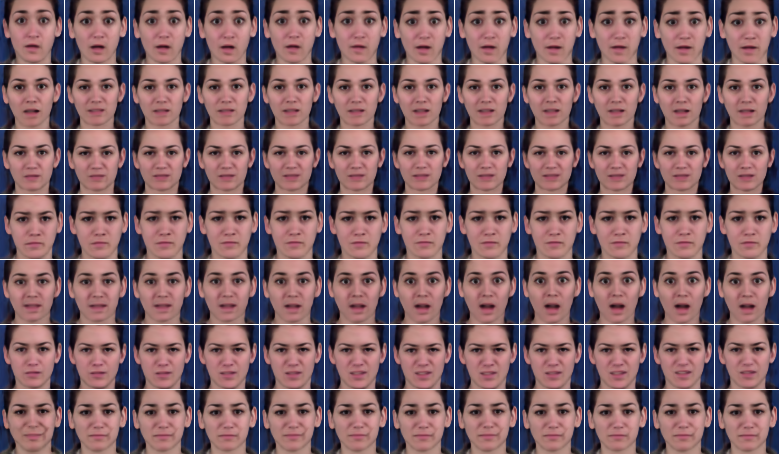}
			\caption{Videos generated by fixing the static representation $\bm{z}_f$  and sampling the dynamic representation $\bm{z}_{1:T}$. Each row shows one generated video sequence. All videos show the same woman, which performs various expressions.}
			\label{m3a}
		\end{subfigure}	\\
		\begin{subfigure}[b]{0.9\textwidth}
			\includegraphics[width=\textwidth]{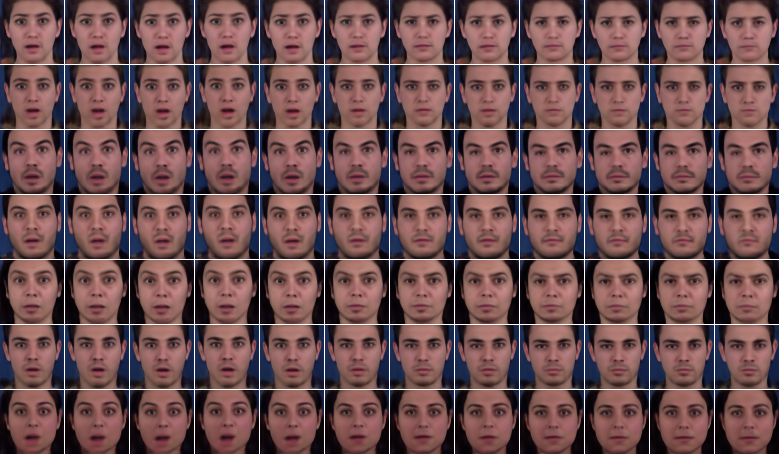}
			\caption{Videos generated by fixing the dynamic representation $\bm{z}_{1:T}$   and sampling the static representation $\bm{z}_f$. Each row shows one generated video sequence. Different videos show different persons, which perform the expression of surprise.}
			\label{m3b}
		\end{subfigure}
		
		\caption{Manipulating video generation on MUG dataset.  }
		\label{m3}
	\end{figure*}
\end{appendices}

\end{document}